\documentclass[runningheads]{llncs}

\usepackage{eccv}

\usepackage{eccvabbrv}

\usepackage{graphicx}
\usepackage{booktabs}

\usepackage[accsupp]{axessibility}  

\usepackage{hyperref}

\usepackage{orcidlink}

\usepackage{capt-of} 
\usepackage{multirow}

\usepackage{amsmath}
\usepackage{amssymb}
\usepackage{mathtools}
\newtheorem{assumption}[theorem]{Assumption}
\usepackage{xspace}

\begin{document}

\title{Unified Backbone Refinement for Diffusion Models via Internal-Latent Analysis} 

\titlerunning{DUNE: Unified Backbone Refinement for Diffusion Models}

\author{Haksoo Lim\inst{1} \and
Myeongjin Lee\inst{1} \and
Wonjoon Chang\inst{1} \and Jaesik Choi\inst{1,2}}

\authorrunning{Lim et al.}

\institute{Korea Advanced Institute of Science and Technology (KAIST), Daejeon 34141, Republic of Korea \and
INEEJI, Seongnam 13558, Republic of Korea}

\maketitle

\begin{figure}[h]
  \centering
  \includegraphics[width=0.95\textwidth]{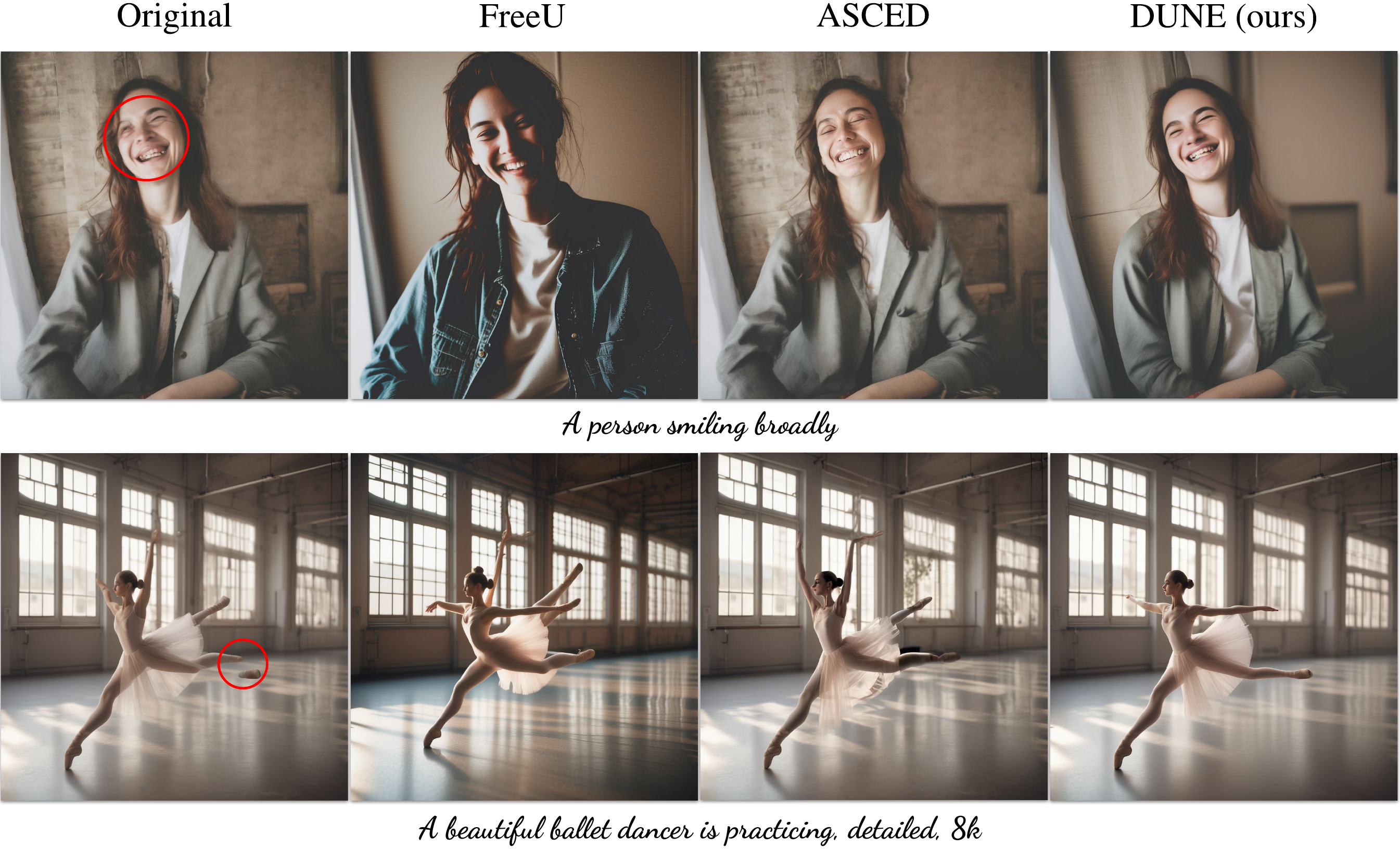}
  \caption{We propose DUNE, a method that freely improves image quality, meaning no additional training or increase of memory. DUNE mitigates visual artifact while preserving semantic consistency.}
  \label{fig:teaser}
  \vspace{-3em}
\end{figure}

\begin{abstract}
Diffusion models have achieved remarkable success across diverse domains, with performance closely related to the denoising backbones that parameterize the score function. In this paper, we present a systematic, phase-aware analysis of diffusion components and show that abrupt, early-stage fluctuations in deep latents are strongly associated with artifacts. Guided by these findings, we introduce \textbf{DUNE} (\textbf{D}iffusion \textbf{U}nified Network refi\textbf{NE}r), a training-free refinement framework that detects abrupt deviations in deep low-noise internal latents using a shared EMA-based criterion, and applies backbone-specific suppression to the detector-selected entries. Although derived from U-Net, the same detect--suppress principle extends naturally to Transformer-based diffusion models by acting on the latents of deep self-attention blocks. 
Extensive experiments across multiple backbones indicate that DUNE improves fidelity while reducing hallucinations, offering new insight into where and when diffusion backbones should be controlled.
\end{abstract}    
\section{Introduction}
\label{introduction}

Recently, diffusion models have achieved remarkable results across various computer vision domains, including image generation, super-resolution, and image editing \cite{ho2020denoising, li2022srdiff, kwon2022diffusion}. Surpassing other generative models, diffusion models have also been successfully extended to other domains such as video generation, voice synthesis, and time series forecasting~\cite{yang2023diffusion_survey}. These models operate through specific procedures, namely a \textit{forward process} and \textit{reverse process} ~\cite{yang2023diffusion_survey,ho2020denoising,scoresde}. During the forward process, the model incrementally adds noise to an input image, finally transforming it into Gaussian noise. In the reverse process, the model learns a gradient of the log-likelihood, also called a \textit{score function}, using specialized architectures,  most commonly U-Net and Transformer variants \cite{ronneberger2015u, dit}. Once trained, diffusion models generate new images by sampling from a Gaussian distribution and applying the learned score function to iteratively denoise the sample.

Across a wide range of applications, the U-Net architecture has typically been employed as the foundational framework for diffusion models~\cite{ho2020denoising, rombach2022high, podell2023sdxl}. A U-Net consists primarily of \textit{downsampling blocks}, a bottleneck region (known as \textit{h-space}), \textit{skip connections}, and \textit{upsampling blocks}. Recent work reveals that skip connections and upsampling blocks play distinct internal roles, emphasizing high-frequency and low-frequency components, respectively~\cite{si2024freeu}. Although existing refinement methods manipulate internal activations or the output of the score network, improper control of these features can lead to deviations from the original image, intensifying visual artifacts—commonly referred to as hallucinations—as well as overly simplified backgrounds and exaggerated contrasts between light and dark regions (see Figure~\ref{fig:teaser}).

In this work, we observe that artifact-associated score dynamics are concentrated in early denoising stages and are reflected in deep internal latents. Through component-wise and phase-aware analysis, we find that detecting and correcting such anomalies in the deepest latent space of U-Net (h-space) substantially improves fidelity while preserving semantics (see Figure~\ref{fig:score_plot}). Motivated by these insights, we introduce \textbf{DUNE}: \textbf{D}iffusion \textbf{U}nified Network refi\textbf{NE}r, a \emph{training-free} refinement framework that suppresses detector-selected internal deviations during the detect phase and optionally adjusts perceptual details in a later phase.

Our analysis reveals that diffusion models already know which parts contain anomalies and can substantially enhance quality by explicitly addressing these areas (cf. Section~\ref{sec:theorems}, ~\ref{sec:proposedmethod}). We partition generation into two phases: a \textit{detect phase}, where anomalies are identified and corrected in the latent representations, and an optional \textit{detail phase}, where internal manipulations can be applied for user-controllable perceptual adjustments such as tone and saturation. We extend the detect--suppress framework to Transformer-based diffusion models by operating on deep self-attention latents that satisfy the same intervention criterion as U-Net h-space: semantic integration with reduced stochastic fluctuation. These improvements incur no additional training or model-specific finetuning, and deliver consistent gains across metrics and models (see Figure~\ref{fig:main}). Across extensive experiments, DUNE consistently improves fidelity and robustness over training-free refinement baselines.

Our contributions are summarized as follows:

\begin{itemize}
\item Through component-wise and phase-specific analysis of U-Net backbones, we show that artifact-associated dynamics are concentrated in early denoising stages and are reflected in deep internal latents, motivating targeted phase-aware control.
\item Based on theoretical and empirical analysis of deep intervention points in U-Net h-space and Transformer self-attention latents, we propose DUNE, which applies a shared EMA-based detector and early masked-correction template with backbone-specific suppressors.
\end{itemize}
\section{Background \& Observation}
\label{sec:observation}

We adopt the DDPM~\cite{ho2020denoising} formulation: a forward 
process corrupts clean data $X_0\sim q(X_0)$ into 
$X_t=\sqrt{\bar\alpha_t}\,X_0+\sqrt{1-\bar\alpha_t}\,\epsilon$ 
with $\epsilon\sim\mathcal{N}(0,\mathbf{I})$, and a score network 
$\epsilon_\theta(X_t,t)$ is trained to predict the injected noise, 
equivalently estimating the score 
function~\cite{ho2020denoising,song2019generative}. Building on 
this foundation, we first review related training-free refinement 
methods, then systematically analyze how individual backbone 
components influence artifact emergence across denoising phases. 
Additional related work appears in 
Appendix~\ref{appendix:relatedworks}.

\subsection{The Roles of U-Net Components across denoising steps}
\label{sec:unetrole}

\begin{figure}[t]
\centering
\includegraphics[width=\linewidth]{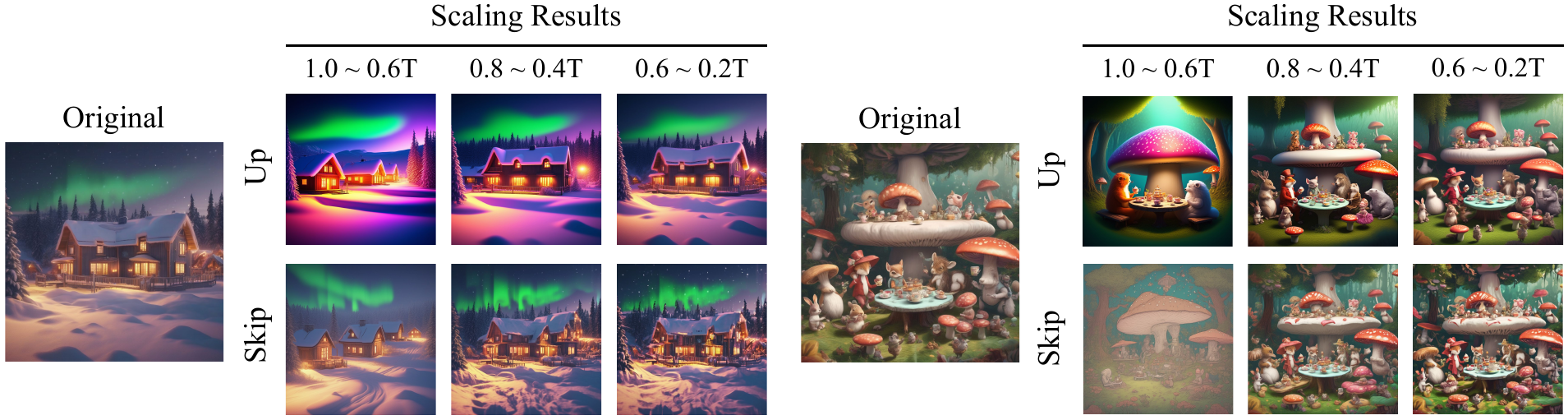}
\caption{Effects of scaling U-Net components across denoising steps. Upsampling blocks affect brightness and tone, while skip connections enhance edge details. In early phases, however, scaling primarily leads to semantic changes.
}
\vspace{-1.0em}
\label{fig:scalephase}
\end{figure}

In this section, we explore the role of U-Net components considering phases of generation, i.e., denoising steps. Specifically, we scale the upsampling blocks and skip connections during particular durations and analyze how each component, along with the timing of scaling, influences the generated images\footnote{In this experiment, we scale input feature maps of all upsampling blocks and do not scale the shallowest skip connections since they are too sensitive, only scaling the remaining skip connections.}. 

Figure~\ref{fig:scalephase} presents the scaling results which indicate three observations. First, scaling upsampling blocks improves the quality of the generated images with smoothened textures, consistent with prior observations in FreeU~\cite{si2024freeu}. In addition, we can also observe that these blocks also contribute to increasing the brightness and tone of the images. Second, while FreeU reports that naively scaling the skip connections show minimal effects, we find that scaling them during mid-to-late phases (e.g., after 0.6T) enhances the detail information and clarifies edges by sharpening faint or indistinct contours in the generated images, suggesting that skip connections are in charge of detail information during mid-to-late phases. Lastly, it is important to note that scaling components during early phases (1.0T $\sim$ 0.6T and 0.8T $\sim$ 0.4T) has a more substantial impact on altering the semantics of the generated images than on fulfilling the specific roles of each component described above. This indicates that, during early phases, the semantic structure formation dominates over the individual functional contributions of the U-Net components.

\textbf{Our hypothesis.} Based on the above observations, unexpected semantic variations (e.g., artifacts) are associated with unstable early-phase latent dynamics in the score network, and naively scaling each latent might exaggerate anomalies and making visual artifacts. We therefore hypothesize that abrupt early-phase variations in component latents are strongly associated with artifact emergence in generated images.
\section{Analysis of Internal Latents}
\label{sec:theorems}

\begin{figure}[t]
\centering
\includegraphics[width=0.8\linewidth]{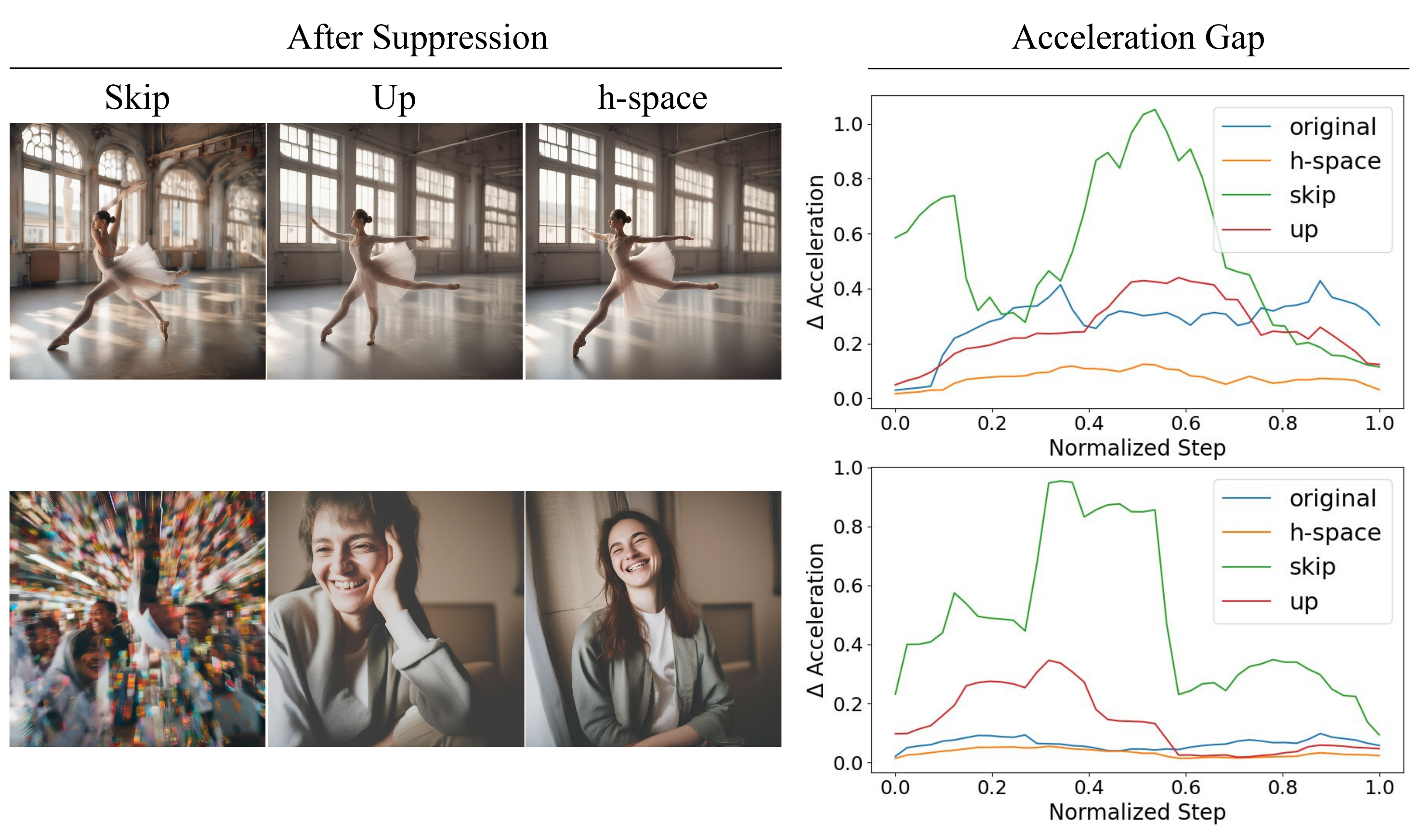}
\caption{Suppression Results with Different Components in U-net Controlling. The change of the score function according to denoising time step is called \textit{acceleration}~\cite{cao2025temporal}.
}
\vspace{-1.0em}
\label{fig:component_plot}
\end{figure}

In this section, we investigate the depth-wise characteristics of the U-Net architecture to support our use of the internal latent spaces—\emph{h-space} for U-Net and the self-attention latents for Transformer backbones. In typical U-Net–based diffusion models, downsampling is performed by convolutional layers (CNNs), halving the spatial dimensions and increasing channel depth~\cite{podell2023sdxl,lcm,arkhipkin2024kandinsky}. The following propositions formally describe how U-Net effectively reduces the noise in diffused images and concentrates the semantic content within the h-space.

\begin{proposition}\label{thm:hspace}
Given an $n\times n$ data matrix $X_0=\{X_0^{i,j}\}_{1\leq i,j\leq n}$ and an $m\times m$ convolutional kernel $K = \{\lambda_{i,j}\}_{1\leq i,j\leq m}$ used in a downsampling layer, the standard deviation of the noise added to diffused data $X_t$ (for $t \in {0,1,...,T}$) is scaled by a factor of $||K||_2=\sqrt{\sum_{i,j}\lambda_{i,j}^2}$ through the convolutional operation.
\end{proposition}

We empirically checked that the mean of kernel norms in each downsampling block are sufficiently small, effectively suppressing the added noise and confirming the theoretical prediction of Proposition~\ref{thm:hspace} (e.g., SDXL shows kernel norms of 0.2036 and 0.1509 in its respective downsampling layers). 

This inherent denoising capability early in the reverse trajectory clarifies \emph{where} to intervene. However, directly correcting these anomalies in shallow layers poses challenges, as simultaneous adjustments of image content and added noise are needed to maintain the normal distribution of the score function outputs (see Equation~\ref{eq:score_loss}). Conversely, corrections applied in the h-space are less problematic, as this deeper latent space inherently condenses semantic information while effectively reducing noise. The first column of Figure~\ref{fig:component_plot} clearly illustrates that corrections in shallower layers, such as skip connections or upsampling blocks, increase score deviations in hallucinated regions, whereas corrections within the h-space significantly mitigate these deviations. Therefore, we operate our detect--suppress mechanism primarily in h-space.

Next, we extend our analysis into Transformer-based diffusion models. The following proposition shows that the latent of self-attention layer has innate denoising capability like h-space in U-Net.

\begin{figure}[t]
\centering
\includegraphics[width=0.8\linewidth]{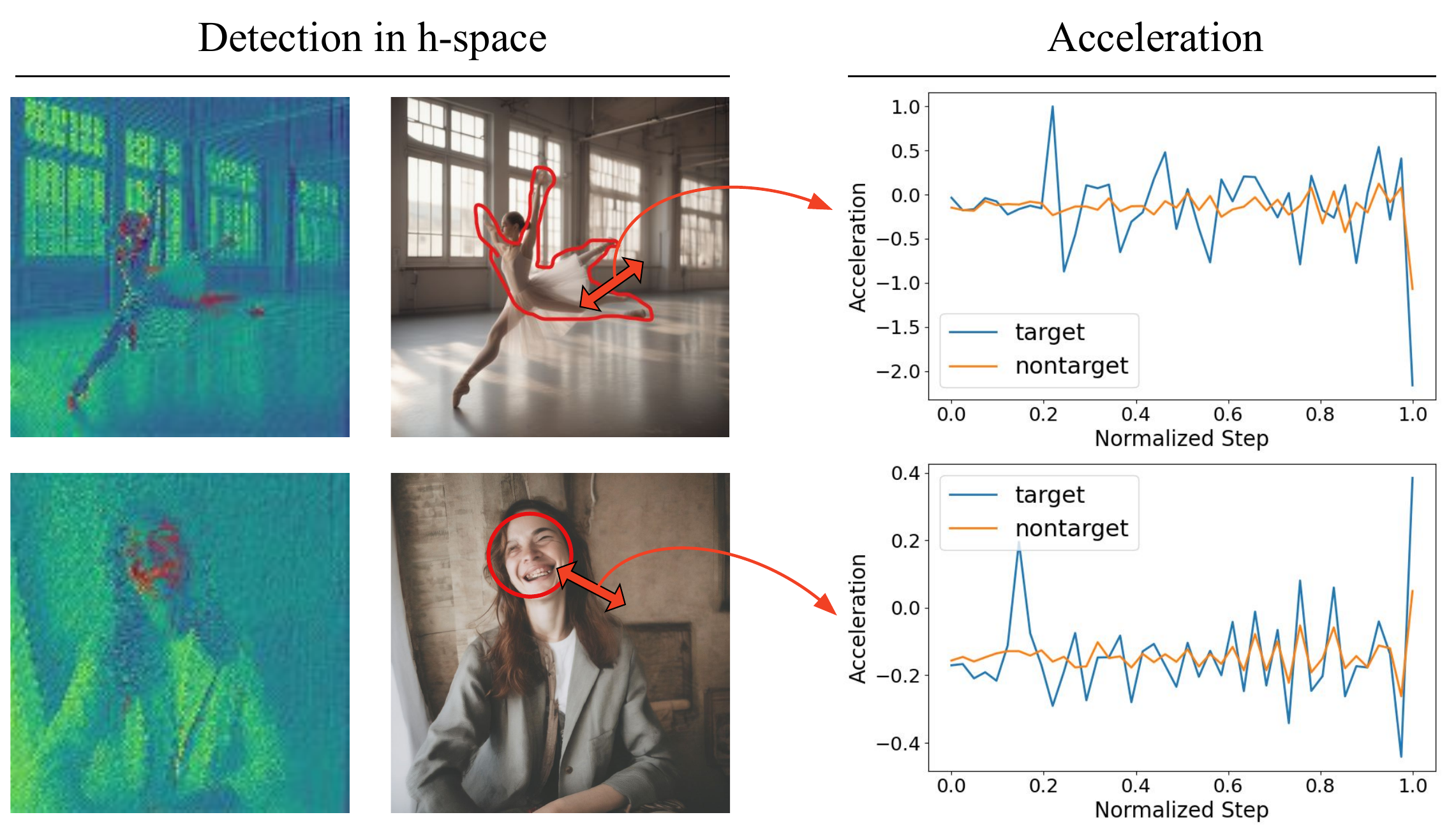}
\caption{We visualize how detector-selected regions align with abrupt score-acceleration patterns. The plot shows the averaged score function across diffusion steps, with ``Target'' and ``Non-target'' representing the inner and outer regions of the red bounding contour, respectively.
}
\vspace{-1.0em}
\label{fig:score_plot}
\end{figure}

\begin{proposition}\label{thm:attn_softmax}
Given tokens $x_j=s_j+\varepsilon_j\in\mathbb{R}^d$ with i.i.d.\ $\varepsilon_j\sim\mathcal{N}(0,\sigma^2 I_d)$ and attention weights $w=\mathrm{softmax}((qK^\top)/\tau)$ ($w_j\ge0$, $\sum_j w_j=1$), the attention output $y=\sum_{j=1}^N w_j x_j$ scales the isotropic noise standard deviation by $\sqrt{\sum_{j=1}^N w_j^2}$, i.e.,
\[
\mathrm{Cov}(y_{\mathrm{noise}})=\sigma^2\!\left(\sum_{j=1}^N w_j^2\right) I_d,
\]
so $\sum_j w_j^2\!\in[1/N,1]$ plays the role of a \emph{data-adaptive kernel norm} (cf.\ Proposition~\ref{thm:hspace} for CNNs).
\end{proposition}


Thus the standard deviation scales by $\sqrt{\sum_j w_{ij}^2}\in[1/\sqrt{N},\,1]$, exactly mirroring the role of the kernel norm in Proposition~\ref{thm:hspace} but now \emph{data–adaptive} via softmax. As the attention distribution becomes more diffuse (higher entropy; e.g., larger temperature), $\sum_j w_{ij}^2 \downarrow$ toward $1/N$ and the denoising effect strengthens. Stacking layers therefore yields a deep token space that is both semantically integrated and noise–reduced, making mid–to–deep self–attention representations the Transformer analogue of U‑Net’s $h$–space and a natural sweet spot for detection and suppression.

The two results above justify our design: (i) downsampling convolutions scale (and often contract) noise toward the U-Net bottleneck, and (ii) self-attention performs noise averaging over tokens before value projection. Operating our detect--suppress mechanism in these deepest latents stabilizes score dynamics (reducing abnormal spikes) and preserves semantics.

Transformer backbones lack explicit skip-/upsampling branches. We therefore define the extension through the intervention criterion. In both architectures, DUNE targets a deep semantically integrated latent where stochastic fluctuations are attenuated.
For Transformers, Proposition~\ref{thm:attn_softmax} justifies mid-to-deep self-attention latents as this operating region, and Appendix~\ref{appendix:attemperical}, \ref{appendix:transformer_layer_ablation} further validates this empirically through latent visualization and target-layer ablation.
\section{Proposed Method}
\label{sec:proposedmethod}

In the previous section, we identified that abrupt changes in the h-space of U-Net components are strongly associated with the emergence of artifacts in generated images. This phenomenon mainly appears at early denoising phases, where semantic properties are determined. Based on these observations, we propose a component-aware refinement framework, \textbf{DUNE}: \textbf{D}iffusion \textbf{U}nified Network refi\textbf{NE}r. Leveraging insights from Section~\ref{sec:observation},~\ref{sec:theorems}, we divide the denoising process into two phases: the \textit{detect phase} (1.0T--0.6T) and the \textit{detail phase} (0.6T--0.0T). This division reflects that refinements in early and later stages predominantly affect global structure and local details, respectively. Note that the detect phase constitutes the core of DUNE and is sufficient for artifact suppression; the detail phase is an optional extension for stylistic control.

The key functionalities and usage of DUNE across the two phases are 
described below. In the detect phase, DUNE identifies and corrects 
potential anomalies within the h-space. The illustrative description 
of this process is shown in Figure~\ref{fig:framework}. The detail 
phase is an optional, user-facing module that enables fine-grained 
control over brightness, tone, and saturation by selectively scaling 
skip connections and upsampling blocks during the later denoising 
steps (cf. Figure~\ref{fig:combined_control_unet}). This step depends 
on the design choice of users, which properties they want to 
emphasize. In our implementation, we adopt the scaling methodology 
from FreeU~\cite{si2024freeu}, restricting it to the detail phase 
to preserve semantic content established during the detect phase. 
However, applying FreeU alone may result in overly saturated images, 
which deviate significantly from prompts intended to produce darker 
or moodier imagery. We therefore investigate each component in greater 
detail, enabling users to finely adjust the clarity and tonal 
qualities of the generated images. Also, since these adjustments are stylistic rather than corrective, we 
exclude the detail phase from all quantitative evaluations and present 
its effects qualitatively in Appendix~\ref{appendix:down}.

\begin{figure}[t]
\centering
\includegraphics[width=0.7\linewidth]{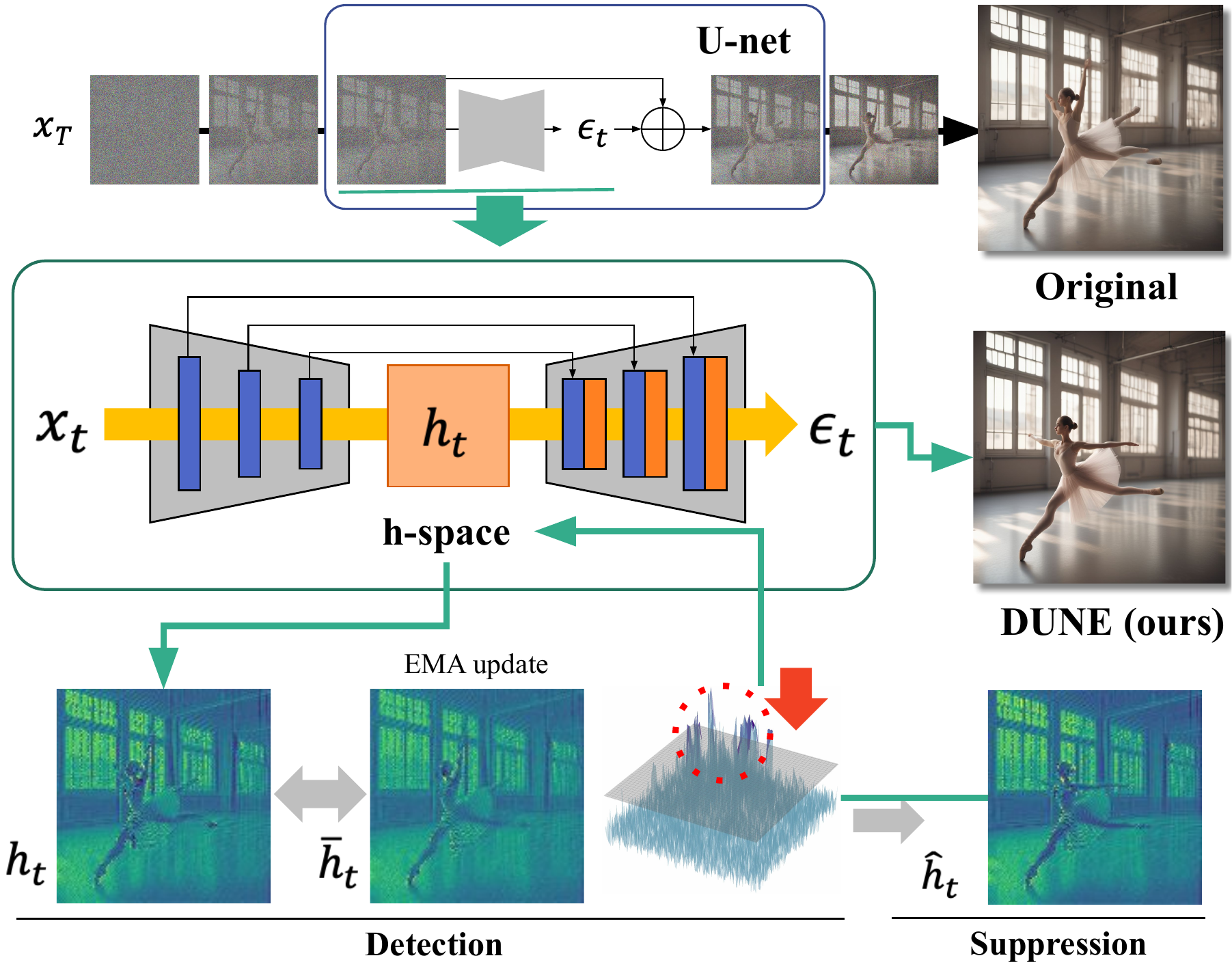}
\caption{Overall framework of DUNE. The h-space is directly matched with self-attention latent in Transformer-based diffusion models.}
\vspace{-1.0em}
\label{fig:framework}
\end{figure}

\subsection{Detection Analysis}\label{sec:analysis}

Several recent studies link visual artifacts in diffusion models to abnormal temporal variations in the score function~\cite{aithal2024understanding,cao2025temporal}. The change of the score function according to denoising time step is called \textit{acceleration}~\cite{cao2025temporal}. Building on these findings, we detect abrupt internal changes in the score network (U-Net and Transformer) over the denoising trajectory. Masks are resized to match the spatial dimensions of the corresponding feature maps, and we compare acceleration between masked and unmasked regions.

During the reverse process to generate images, we obtain latent variable $\textbf{h}_t$ at normalized timestep $t \in [1,T]$, where $T$ denotes the total number of inference steps. To detect abrupt changes in latent features, we compute the exponential moving average (EMA) of a scaled latent variable, which represents the estimation. Motivated by prior observations that artifact regions exhibit abnormal score dynamics~\cite{aithal2024understanding,cao2025temporal}, and given that the score function in diffusion models is defined as $
s_{\theta}(t,\textbf{x}_t) = -\frac{\epsilon_{\theta}(t,\textbf{x}_t)}{\sigma_t}$,
we apply detection to \(z_t=-h_t/\sigma_t\), the score-consistent normalization of \(h_t\), to factor out timestep-dependent amplitude drift and isolate abrupt internal deviations. Consequently, our EMA is defined as:
$
\bar{\textbf{z}}_t \;=\; \gamma \bar{\textbf{z}}_{t+1} + (1-\gamma) \textbf{z}_t.
$
We use the standard 
value $\gamma=0.7$ throughout all experiments without per-backbone 
tuning. The EMA is initialized as 
$\bar{\textbf{z}}_T = \textbf{z}_T$ at the first inference step.

We then measure deviations between each scaled latent ${\textbf{z}}_t$ and its corresponding EMA $\bar{\textbf{z}}_{t+1}$ at each step. Because these values differ according to the brightness of individual images, we adopt a log-ratio to facilitate consistent comparisons across images:
\[
\Delta = \log\mid\frac{\textbf{z}_t}{\bar{\textbf{z}}_{t+1}}\mid.
\]

Since $\Delta$ is less sensitive to image-specific scale and brightness variations, we identify a quantile threshold $\lambda$ corresponding to the $p\%$ percentile, thereby isolating the top $(1 - p)\%$ most deviated features. The resulting binary mask $M$ is defined as $M = (\Delta > \lambda)$,
highlighting anomalous regions (see second column in Figure~\ref{fig:score_plot}). Figure~\ref{fig:score_plot} also confirms that detected regions exhibit strong acceleration fluctuations concentrated in the first 40\% of steps (the detect phase), and that the same logic applies to Transformer self-attention latents.

\subsection{Suppression Analysis}
\label{sec:sup-analysis}


We use masked EMA suppression as the canonical correction objective for detector-selected low-noise latents. Transformer latents instantiate this objective through masked EMA blending. U-Net h-space contains low-resolution channels with specialized semantic activations, so DUNE uses channel-aware masked scaling to shrink detected anomalous channels while preserving channel specialization. Appendix~\ref{appendix:snr_proof} analyzes this U-Net instantiation on detector-selected low-SNR subsets and relates it to EMA-style suppression.

Let $\textbf{h}_t$ denote the target internal latent at timestep $t$ (U-Net $h$-space or Transformer self-attention latent),
and let $M_t$ be the detected anomaly mask.
We apply a masked correction in a unified gated form:
\begin{equation}
\hat{\textbf{h}}_t=(1-M)\odot\textbf{h}_t \;+\; M_t \odot \mathcal{S}_{\text{bb}}(\textbf{h}_t,\bar{\textbf{h}}_t;\kappa),
\label{eq:unified_suppress}
\end{equation}
where $\mathcal{S}_{\text{bb}}$ denotes a backbone-specific suppression operator, $\bar{\textbf{h}}_t := -{\sigma_t}\bar{\textbf{z}}_t$ is the de-scaled EMA reference used in detection and $\kappa\in[0,1]$ controls correction strength.

For U-Net-based diffusion models, we suppress detector-selected h-space outliers because our analysis identifies this region as a stable intervention point where artifact-associated deviations are concentrated. Since each $\textbf{h}_t$ consists of multiple channels with relatively low spatial resolution (e.g., $1280\times32\times32$ for SDXL), where each channel encodes distinct semantic information. Hence, channels must be treated independently. We compute the absolute value of the channel-wise mean of $\textbf{h}_t$, denoted by $\textbf{n}_t$ (shape of $1280\times1\times1$). Using a scaling factor $\kappa$, the corrected latent representation becomes $\mathcal{S}_{\text{UNet}}({\textbf{h}}_t,\bar{\textbf{h}}_t;\kappa) \;=\; \kappa\,(n_t \cdot {\textbf{h}}_t)$.


For Transformer-based methods, latents are patchified and self-attention mixes tokens, so simple masked scaling is less stable. We therefore blend masked features with the EMA reference:
$
\mathcal{S}_{\text{Tr}}({\textbf{h}}_t,\bar{\textbf{h}}_t;\kappa)
\;=\; \kappa {\textbf{h}}_t + (1-\kappa)\bar{\textbf{h}}_t.
$
We report the performance of our detect--suppression methods for both U-Net and Transformer backbones in Figure~\ref{fig:main}.

Although the algebraic forms differ, both suppressors implement the same local stability objective in detector-selected low-noise latents. The Transformer case admits direct EMA blending, whereas the U-Net case uses a bounded channel-aware instantiation of EMA-style suppression to avoid disrupting specialized h-space channels.

\subsection{SNR Analysis of the Detect--Suppress Phase}
\label{sec:snr}

We refine the theoretical role of the detect--suppress phase by explicitly analyzing the
\emph{signal-to-noise ratio (SNR)} in the internal latent that DUNE operates on.
We model the scaled latent $\textbf{z}_t := -\frac{\textbf{h}_t}{\sigma_t}\in\mathbb{R}^{d}.$ as a sum of a slowly-varying semantic component and a stochastic component:
\begin{assumption}
\label{assump:decomp}
For each timestep $t$, the scaled latent admits a decomposition
${\textbf{z}}_t = s_t + n_t$, where \(s_t:=\mathbb{E}[z_t\mid\mathcal{F}_t]\) and \(n_t:=z_t-s_t\), 
In here, $s_t$ denotes a semantic signal, $n_t$ denotes a stochastic component
(noise / abnormal spikes) and \(\mathcal{F}_t\) contains the conditioning, timestep, and slowly varying semantic state.
Then \(\mathbb{E}[n_t\mid\mathcal{F}_t]=0\) by construction.
\end{assumption}

\begin{assumption}
\label{assump:slow_drift}
There exists $\delta\ge 0$ such that $\|s_{t+1}-s_t\|_2 \le \delta$ for all $t$ in the detect phase.
\end{assumption}

\begin{definition}
\label{def:snr}
We define the latent SNR at timestep $t$ as
\begin{equation}
\mathrm{SNR}(z_t) \;:=\;
\frac{\mathbb{E}\|s_t\|_2^2}{\mathbb{E}\|n_t\|_2^2}.
\label{eq:snr_def}
\end{equation}
\end{definition}

Let $M_t\in\{0,1\}^{d}$ denote a binary mask (broadcastable to the latent shape) produced by the detect
step. Define masked/unmasked parts by $a_{t,M}:=M_t\odot a_t$ and $a_{t,\neg M}:=(1-M_t)\odot a_t$. Then we analyze a unified suppression operator:
\begin{equation}
\hat{\textbf{z}}_t
\;=\;
(1-M_t)\odot \textbf{z}_t \;+\; M_t\odot\big(\kappa \textbf{z}_t + (1-\kappa)\bar{\textbf{z}}_t\big),
\qquad \kappa\in[0,1].
\label{eq:generic_suppress}
\end{equation} In our U-Net implementation, the masked channels are shrunk with a channel-aware
factor, which can be interpreted as a per-channel version.

Next, define the \emph{noise concentration} and \emph{signal concentration} within the detected mask by
\begin{center}
  \begin{minipage}{0.45\textwidth}
    \begin{equation}
      \eta_t := \frac{\mathbb{E}\|n_{t,M}\|_2^2}{\mathbb{E}\|n_t\|_2^2} \in [0,1]
      \label{eq:eta_def}
    \end{equation}
  \end{minipage}
  \begin{minipage}{0.45\textwidth}
    \begin{equation}
      \rho_t := \frac{\mathbb{E}\|s_{t,M}\|_2^2}{\mathbb{E}\|s_t\|_2^2} \in [0,1]
      \label{eq:rho_def}
    \end{equation}
  \end{minipage}
\end{center}
which capture how much noise and signal energy, respectively, fall within $M_t$. The following theorem shows SNR gain of EMA blending:

\begin{theorem}
\label{thm:snr_ema}
Consider Eq.~\eqref{eq:generic_suppress} and define the EMA estimation error
$e_t := \bar{\textbf{z}}_{t+1}-s_t$.
Under Assumption~\ref{assump:decomp} and assuming $\mathbb{E}\langle n_t,e_t\rangle=0$,
the post-suppression SNR satisfies
\begin{equation*}
\frac{\mathrm{SNR}(\hat z_t)}{\mathrm{SNR}(z_t)}
\;\ge\;
\frac{1}{1-(1-\kappa^2)\eta_t + (1-\kappa)^2 \varepsilon_t},
\qquad
\varepsilon_t := \frac{\mathbb{E}\|e_{t,M}\|_2^2}{\mathbb{E}\|n_t\|_2^2}.
\label{eq:snr_gain_ema}
\end{equation*}
Consequently, $\mathrm{SNR}(\hat{\textbf{z}}_t)>\mathrm{SNR}({\textbf{z}}_t)$ whenever
\begin{equation*}
(1-\kappa^2)\eta_t \;>\; (1-\kappa)^2 \varepsilon_t.
\label{eq:snr_improve_condition}
\end{equation*}
\end{theorem}

The following corollary shows that suppressing anomalies in deep latents bounds the resulting score perturbation:

\begin{corollary}
\label{cor:score_stability}
Let the remaining mapping from the target latent to noise prediction be
$\epsilon_\theta(\cdot,t)=g_t(\cdot)$ and assume $g_t$ is $L_t$-Lipschitz.
Then the change in the predicted score satisfies
\begin{equation*}
\|\; s_\theta(t,\hat{\mathbf{x}}_t) - s_\theta(t,{\mathbf{x}}_t)\;\|_2
\;\le\;
\frac{L_t}{\sigma_t}\,\|\hat{\mathbf{h}}_t-{\mathbf{h}}_t\|_2
\;\propto\;
\frac{L_t(1-\kappa)}{\sigma_t}\,\|M_t\odot({\mathbf{z}}_t- \bar{\mathbf{z}}_t)\|_2.
\label{eq:score_stability}
\end{equation*}
This provides a local stability argument for why suppressing detector-selected residuals can reduce score perturbations.
\end{corollary}


In the outlier regime selected by our detector, the following 
lemma shows that replacing masked EMA blending with a masked 
channel-wise scaling is justified: the only discrepancy consists 
of (i)~a coefficient mismatch $\tilde\kappa-\alpha_t$, 
which we empirically confirm is near zero at our chosen 
hyperparameters, and 
(ii)~a provably small EMA residual suppressed by the detection 
threshold $\lambda$.

\begin{lemma}
Consider the naive EMA blending operator:
$
\mathbf u_t^{\mathrm{EMA}} := \kappa \mathbf z_t + (1-\kappa)\bar{\mathbf z}_t,\quad \kappa\in[0,1].
$
For U-Net, define the channel statistic $n_t\in\mathbb R^{C\times 1\times 1}$ as in Sec.~\ref{sec:sup-analysis}
and let the channel-aware scaling coefficient be
$
\alpha_t := \kappa\, n_t.
$
Define the U-Net scaling output in the scaled-latent space as
$
\mathbf u_t^{\mathrm{UNet}} := \alpha_t \odot \mathbf z_t.
$
Let $\tilde\kappa := 1-\gamma(1-\kappa)=\kappa+(1-\kappa)(1-\gamma)$.
Then for any $p\in[1,\infty]$,
\[
\|M_t\odot(\mathbf u_t^{\mathrm{EMA}}-\mathbf u_t^{\mathrm{UNet}})\|_p
\;\le\;
\|M_t\odot(\tilde\kappa-\alpha_t)\odot \mathbf z_t\|_p
\;+\;
(1-\kappa)\gamma e^{-\lambda}\,\|M_t\odot \mathbf z_t\|_p.
\]
\end{lemma}

\begin{table*}[t]
    \centering
    \caption{Quantitative comparison of Original vs. Fixed models. Arrows indicate the preferred direction for each metric. ``--'' denotes unsupported backbones; N/A indicates that the method produced pure noise (cf. Figure~\ref{fig:fig16}).}
    \label{tab:metric}
    \vspace{-0.5em}
    \renewcommand{\arraystretch}{1.1} 
    \setlength{\tabcolsep}{5pt}     
    
    \resizebox{0.9\textwidth}{!}{ 
    \begin{tabular}{l c c c c c c}
      \toprule
      \textbf{Methods} & \textbf{Metric} & \textbf{SDXL} & \textbf{LCM} & \textbf{Kandinsky 3} & \textbf{PixArt-$\Sigma$} & \textbf{Hunyuan-DiT} \\
      \midrule
      \multirow{3}{*}{Original}
        & FID $\downarrow$         & 18.86 & 22.53 & 21.36 & 28.75 & 30.82 \\
        & HADM $\downarrow$         & 0.227 & 0.152 & 0.250 & 0.316 & 0.294 \\
        & CLIP sim. $\uparrow$    & 1.0 & 1.0 & 1.0 & 1.0 & 1.0 \\
      \midrule
      \multirow{3}{*}{FreeU}
        & FID $\downarrow$         & 31.22 & 27.57 & --    & -- & -- \\
        & HADM $\downarrow$         & 0.351 & 0.237 & -- & -- & -- \\
        & CLIP sim. $\uparrow$    & 0.781 & 0.756 & -- & -- & -- \\
      \midrule
      \multirow{3}{*}{ASCED}
        & FID $\downarrow$         & 22.07 & 27.60 & 22.44 & N/A & N/A \\
        & HADM $\downarrow$         & 0.227 & 0.151 & 0.252 & N/A & N/A \\
        & CLIP sim. $\uparrow$    & 0.945 & 0.995 & 0.924 & N/A & N/A \\
      \midrule
      \multirow{3}{*}{PAG}
        & FID $\downarrow$         & 21.72 & 35.87 & -- & 32.61 & -- \\
        & HADM $\downarrow$         & 0.210 & 0.123 & -- & 0.284 & -- \\
        & CLIP sim. $\uparrow$    & 0.821 & 0.362 & -- & 0.840 & -- \\
      \midrule
      \multirow{3}{*}{\textbf{DUNE}}
        & FID $\downarrow$         & \textbf{18.75} & \textbf{22.41} & \textbf{20.76} & \textbf{27.80} & \textbf{30.67} \\
        & HADM $\downarrow$         & \textbf{0.074} & \textbf{0.052} & \textbf{0.095} & \textbf{0.254} & \textbf{0.283} \\
        & CLIP sim. $\uparrow$    & 0.925 & 0.940 & 0.933 & 0.952 & 0.988 \\
      \bottomrule
    \end{tabular}}
    \vspace{-1.0em}
\end{table*}

\section{Experiments}
\label{experiments}

\begin{figure*}[ht!]
\centering
\includegraphics[width=0.85\linewidth]{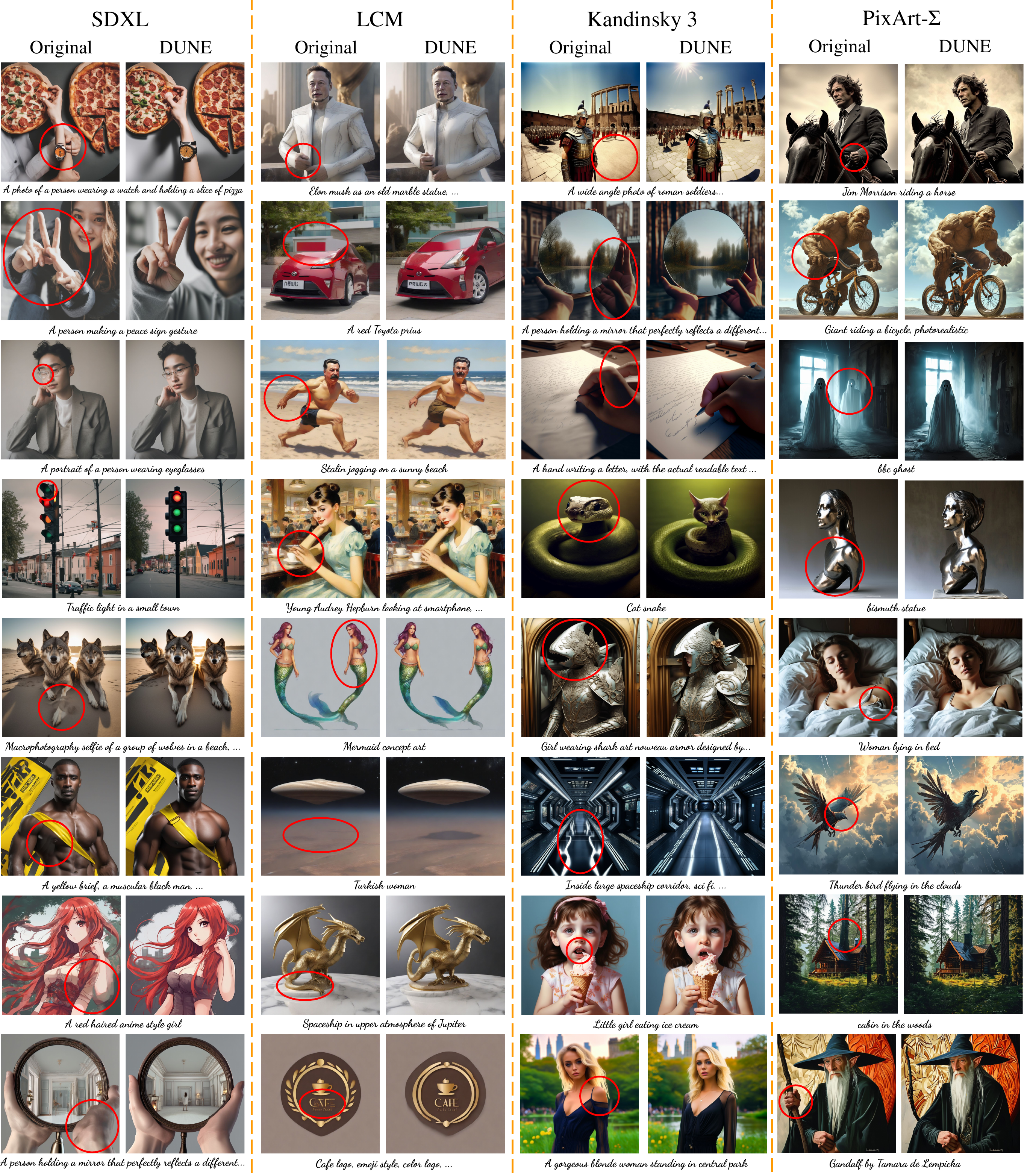}
\caption{An experimental result on high-resolution text to image generation. We give remained Hunyuan-DiT figures in Appendix~\ref{appendix:remained_figures}.}
\label{fig:main}
\vspace{-2.0em}
\end{figure*}

\subsection{Experimental Environments}

We evaluate DUNE across various high-resolution diffusion models. We select widely-adopted U-Net-based models (Stable Diffusion XL (SDXL)~\cite{podell2023sdxl}, Latent Consistency Model (LCM)~\cite{lcm}, and Kandinsky 3~\cite{arkhipkin2024kandinsky}) and Transformer-based models (PixArt-$\Sigma$~\cite{pixart_sigma} and Hunyuan-DiT~\cite{hunyuan}). Our implementations use the publicly available \texttt{diffusers}\footnote{https://github.com/huggingface/diffusers.git} library. Both qualitative and quantitative analyses are provided to demonstrate that DUNE consistently refines and enhances generated images.

We compare against three representative training-free model-augment refinement methods: FreeU~\cite{si2024freeu}, ASCED~\cite{cao2025temporal} and PAG~\cite{PAG}. FreeU reweights backbone and skip features of a pre-trained diffusion U-Net at inference time. ASCED analyzes temporal score dynamics to detect artifact-prone regions and injects corrective noise during sampling. PAG controls internal latent of self-attention layer of diffusion backbones. We give detailed description of these models in Appendix~\ref{appendix:relatedworks}. Note that we use official implementations in \texttt{diffusers} and the authors’ code for ASCED\footnote{\url{https://github.com/YuCao16/ASCED.git}}, following the authors' published default settings.

For quantitative evaluation, we assess the overall generation quality by using the Fréchet Inception Distance (FID) metric\footnote{https://github.com/boomb0om/text2image-benchmark} calculated on 5,000 images generated from COCO prompts~\cite{COCOdataset}. 
We also report HADM-L~\cite{HADM}, a detection-based metric that localizes fine-grained anatomical defects (e.g., extra fingers, distorted joints). For HADM analysis, we generate 5,000 images from human-related prompts by using the human-related annotation files from official COCO resource, following HADM paper~\cite{HADM}. Lastly, since we observed existing refinement methods fail to preserve consistent semantic (e.g., FreeU), we calculate clip similarity to check original-to-refined semantic consistency: cosine similarity of CLIP embeddings between original and refined images. Details of all metrics appear in Appendix~\ref{appendix:metric}. Note that unless otherwise stated, all quantitative results report the detect phase only (detect–suppress without detail-phase reweighting), ensuring a fair comparison with baselines that do not employ component scaling~\footnote{We adopt this 
over text--image CLIP score, as we found that overly saturated outputs can paradoxically achieve higher CLIP scores, making it unreliable for evaluating refinement quality. 
Further analysis is provided in Appendix~\ref{appendix:metric}.}.

\subsection{Enhancing Quality with DUNE}  




\begin{figure*}[t]
    \centering
    
    \begin{minipage}[c]{0.5\textwidth}
        \centering
        \captionof{table}{User survey result for qualitative comparison of Original (red) vs. DUNE outputs (blue).}
        \vspace{1.0em}
    \renewcommand{\arraystretch}{1.1} 
    \setlength{\tabcolsep}{5pt}     
        \resizebox{0.95\linewidth}{!}{
        \begin{tabular}{lccc}
          \toprule
          \textbf{Model} & \textbf{Dataset} & \textbf{Original} & \textbf{DUNE} \\
          \midrule
          \multirow{2}{*}{DDPM} & CelebA-HQ & 6.60\% & \textbf{93.40\%} \\
                                & Bedroom   & 15.28\% & \textbf{84.72\%} \\
          \midrule
          \multirow{2}{*}{VE}   & FFHQ      & 11.51\% & \textbf{88.49\%} \\
                                & Church    & 16.67\% & \textbf{83.33\%} \\
          \bottomrule
        \end{tabular}}


        \vspace{.5em} 

        \includegraphics[width=0.8\linewidth]{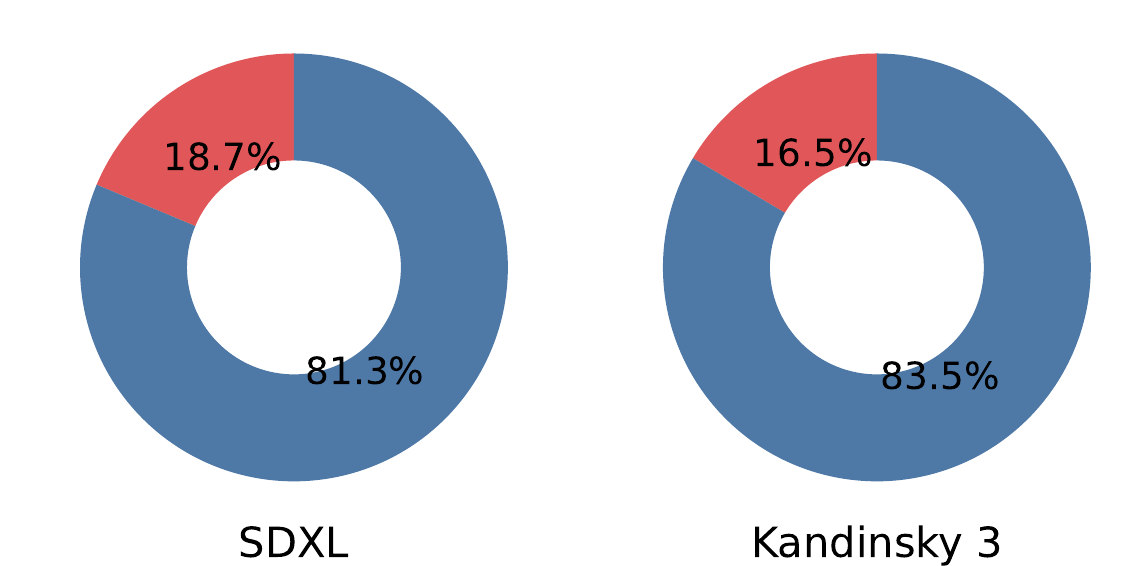}
        \label{tab:usersurvey}
        \label{figure:piechart}
    \end{minipage}%
    \hfill
    \begin{minipage}[c]{0.45\textwidth}
        \centering
        \includegraphics[width=\linewidth]{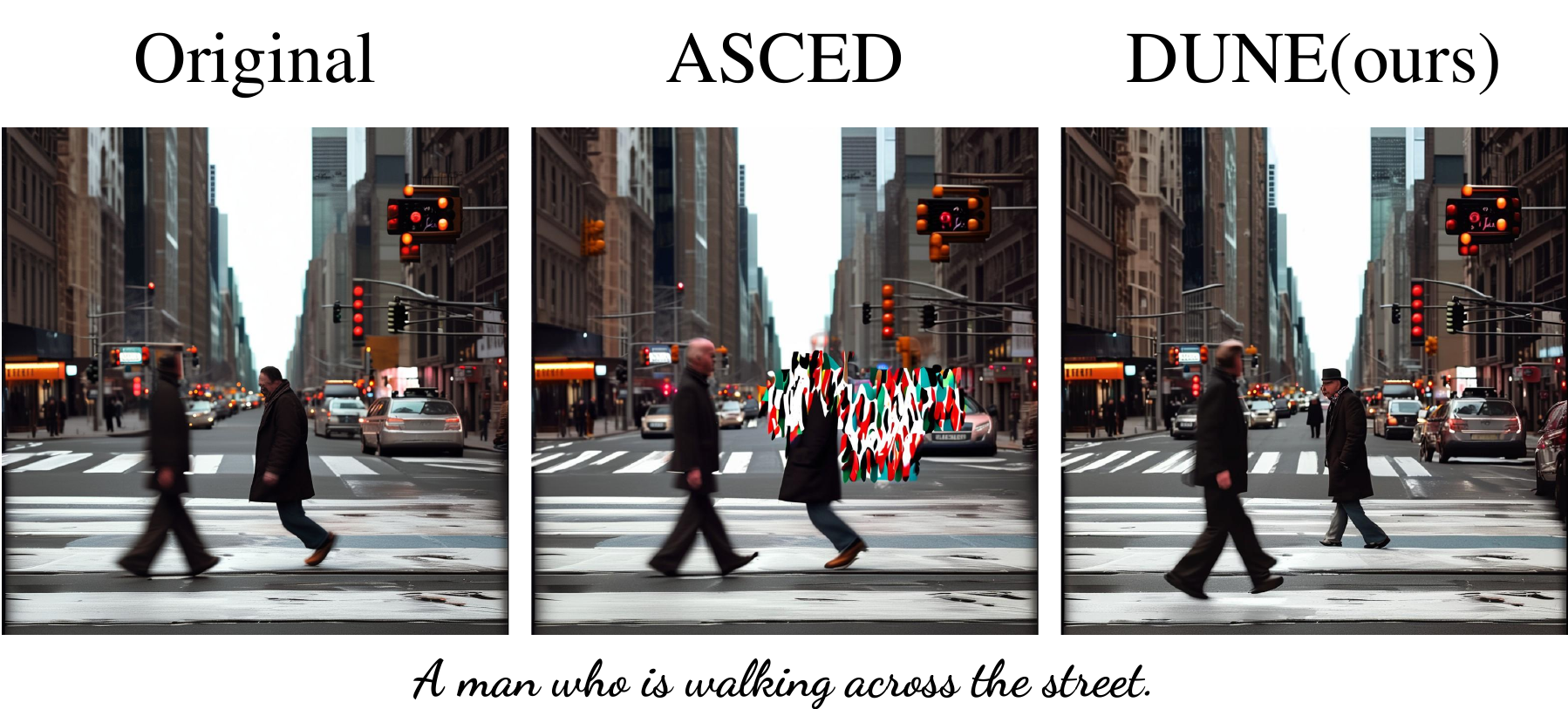}
        \vspace{-1.5em} 
        \captionof{figure}{Visual comparison of ASCED and DUNE on PixArt-$\Sigma$.}
        \label{fig:fig16}
        
        \vspace{.5em} 
        
        \includegraphics[width=0.8\linewidth]{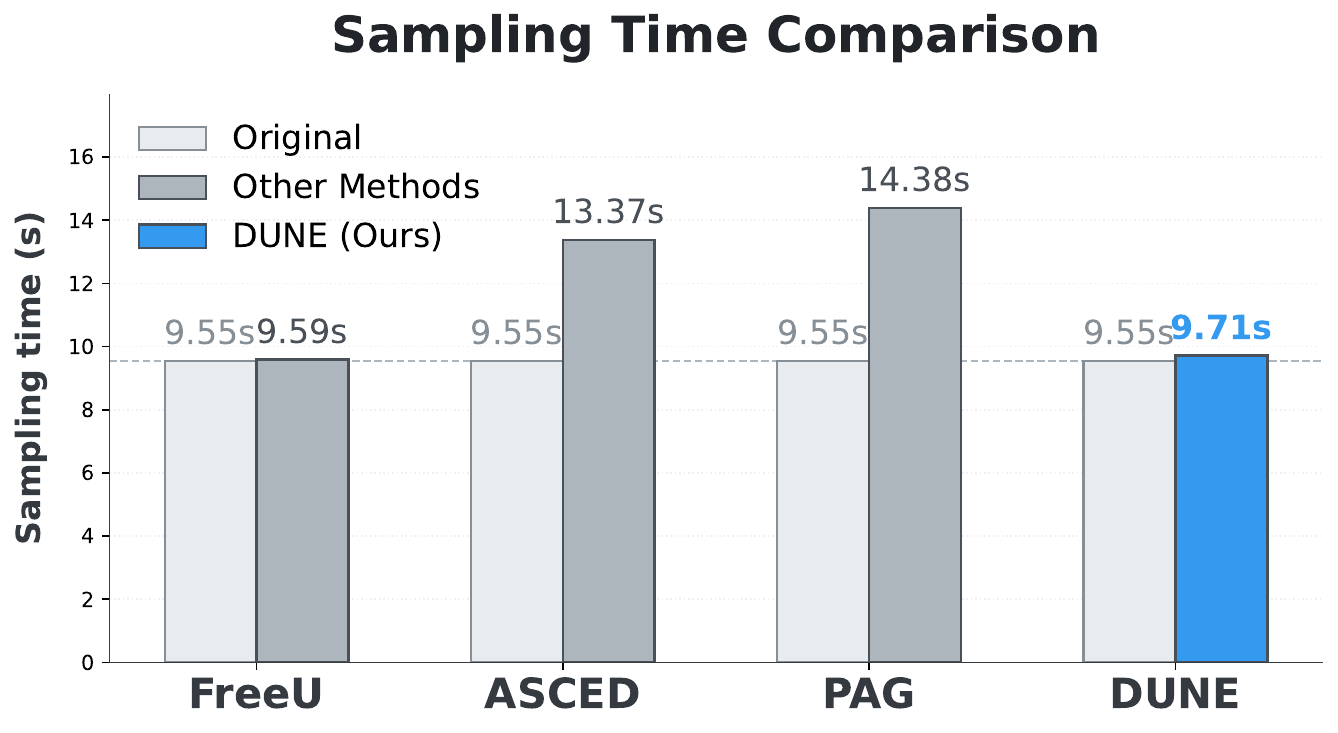}
        \vspace{-.5em}
        \captionof{figure}{Sampling time comparison on SDXL.}
        \label{fig:fig16_1}
    \end{minipage}
    \vspace{-1.em}
\end{figure*}

We first verify that DUNE suppresses detector-selected unstable regions. Figure~\ref{fig:acceleration_dynamics} (extending Figure~\ref{fig:score_plot}) depicts score-acceleration in target/non-target regions before/after DUNE. DUNE attenuates the acceleration spikes in the detector-selected target region, while the non-target region remains largely unchanged. We quantify this effect using normalized acceleration amplification (NAA) and excess acceleration gap reduction (EAGR):
{\small
\begin{align*}
    \mathrm{NAA}(M)=
\mathbb{E}_{t}\!\left[
\frac{\sqrt{\mathbb{E}_{c,i\in M}\,a_t(c,i)^2}}
     {\sqrt{\mathbb{E}_{c,i\notin M}\,a_t(c,i)^2}+\epsilon}
\right],\;\;\;\;
\mathrm{EAGR}=1-\frac{\mathrm{NAA}_{D}^{TopK}-\mathrm{NAA}_{D}^{Rand}}{\mathrm{NAA}_{O}^{TopK}-\mathrm{NAA}_{O}^{Rand}+\epsilon}.
\end{align*}}
Here \(a_t\) denotes score acceleration, \(D/O\) denote DUNE/original sampling, \(\mathrm{TopK}\) is the detector mask, and \(\mathrm{Rand}\) is a density-matched random mask. DUNE reduces the excess target-vs-random acceleration gap by \(24.9\%\) on SDXL and \(18.1\%\) on PixArt-\(\Sigma\); random-mask control shows no comparable reduction.

These diagnostics are consistent with our intervention-location ablations. In U-Net backbones, Figure~\ref{fig:component_plot} shows that suppressing shallower skip/upsampling features can increase score deviations in hallucinated regions, whereas h-space suppression reduces them. For Transformer backbones, Table~\ref{tab:pixart_layer_ablation} shows that the \(3/4\)-depth self-attention layer gives the best FID. Table~\ref{tab:random_mask_ablation} further shows that density-matched random h-space suppression degrades FID, indicating that the gain comes from detector-selected intervention in deep low-noise latents.

\begin{figure}[t]
  \centering
  \vspace{-1.em}
  \begin{subfigure}{0.75\linewidth}
    \centering
    \includegraphics[width=\linewidth]{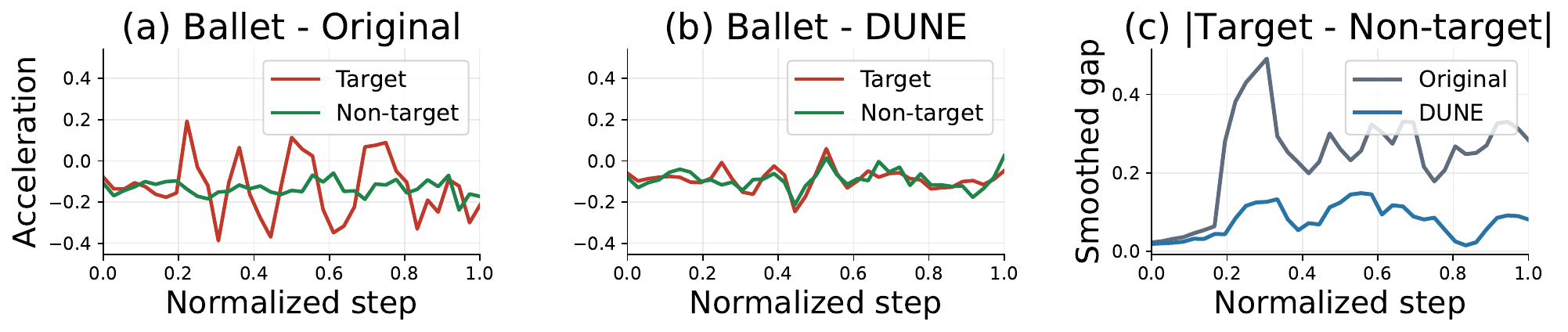}
    \vspace{-2.em}
    \label{fig:ballet}
  \end{subfigure}
  \begin{subfigure}{0.75\linewidth}
    \centering
    \includegraphics[width=\linewidth]{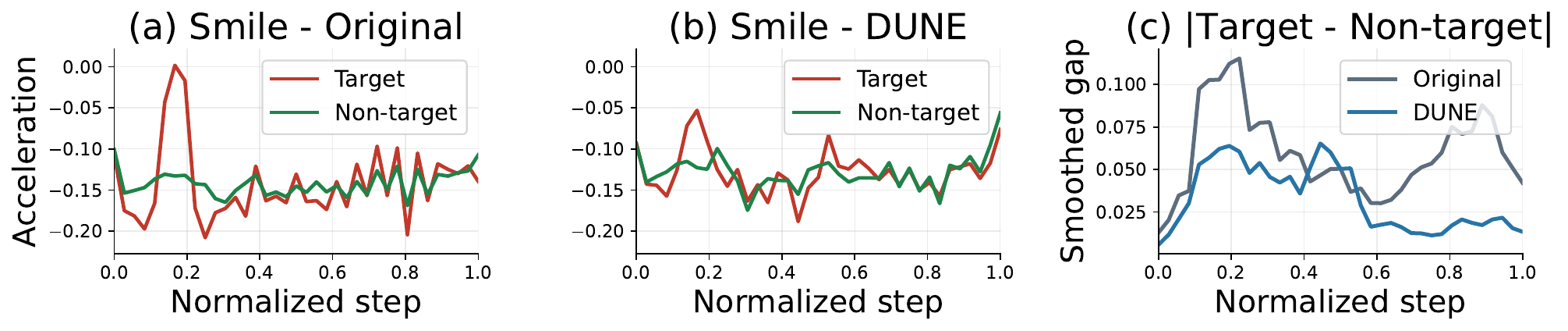}
    \label{fig:smile}
  \end{subfigure}
  \vspace{-2.em}
  \caption{Before/after score acceleration. Target denotes the detector-selected mask and non-target denotes its complement. Curves omit the final 10 denoising steps and are normalized only for visualization.}
  \label{fig:acceleration_dynamics}
  \vspace{-1.em}
\end{figure}

Table~\ref{tab:metric} shows that DUNE exhibits no artifact--fidelity trade-off in our evaluation: it improves both FID and HADM over the vanilla generator on all five tested backbones. On SDXL, every supported baseline worsens FID, whereas DUNE improves FID by \(0.11\) and gives the largest HADM drop (\(0.227\!\to\!0.074\)). Unsupported entries are reported explicitly because the corresponding official/default implementations are unavailable for those backbones: FreeU lacks official/default settings for Kandinsky~3, PixArt-\(\Sigma\), and Hunyuan-DiT, while PAG lacks official/default settings for Kandinsky~3 and Hunyuan-DiT. DUNE also maintains high original-to-refined CLIP similarity, indicating low global semantic drift from the input generation. Overall, Table~\ref{tab:metric} shows that DUNE reduces localized artifacts while preserving the global content of the original sample.

Notably, ASCED, which edits the score trajectory directly, performs poorly on Transformer-based high-resolution text-to-image settings, introducing noisy artifacts (cf. Figure~\ref{fig:fig16}). This suggests that direct score correction can deviate the denoising path and that Transformers are sensitive to abrupt trajectory changes and channel-agnostic masking; DUNE's latent-level, backbone-aware suppression avoids these failure modes. Even though PAG achieves better HADM than original images, PAG shows overly saturated images with simplified edges, as reported in~\cite{PAG}. We also give visual comparison in Appendix~\ref{appendix:refinement}.

Finally, we summarize the practical settings of DUNE. We sweep the detection percentile \(p\) and suppression factor \(\kappa\) on LCM and Kandinsky~3, computing FID on \(1{,}000\) images for each pair
\((p,\kappa)\in\{0.3,0.5,0.7,0.9\}\times\{0.1,0.3,0.5,0.7,0.9\}\).
Figure~\ref{fig:ablation_pick} in Appendix~\ref{appendix:sensitivity} shows the resulting heatmaps. Across both backbones, \(p=0.9\) gives the best FID, so we fix \(p=0.9\) for all main experiments. The suppression factor \(\kappa\) is fixed once per backbone, with values summarized in Table~\ref{tab:hyperparam}; no per-prompt tuning is used. For Transformer backbones, the target layer is fixed at \(3/4\) depth, following the depth ablation in Table~\ref{tab:pixart_layer_ablation}.

\subsection{Qualitative Evaluation}

Figure~\ref{fig:main} compares the generated images from the original generation process (left column) and the proposed DUNE (right column) across various diffusion models. In the original generations, diffusion models often struggle with fine structures (e.g., fingers, teeth) and may introduce extra parts, producing unrealistic complexity. DUNE suppresses these extraneous components and yields more coherent, detailed structures. When prompts are challenging or intentionally unusual (e.g., cat snake), DUNE effectively addresses hallucinations, removing artifacts while preserving semantic consistency with the original generation.

Appendix~\ref{appendix:refinement} provides qualitative comparisons with other baselines. Even with the default hyperparameters reported in their original papers, FreeU and PAG produce overly saturated images with simplified edges, losing the original semantic intent, as the authors mentioned~\cite{si2024freeu,PAG} (see also Appendix~\ref{appendix:freeu}). ASCED, which directly 
modifies the score trajectory, tends to inject noise that degrades local structure, particularly in high-resolution Transformer-based generations. In contrast, DUNE reduces artifacts without inducing 
semantic drift or over-saturation by operating on deep latents rather than the score output itself.

We also conduct a user study with 50 volunteers. For breadth, we include unconditional models (DDPM and VE) on CelebA-HQ, Bedroom, FFHQ, and Church, as well as high-resolution text-to-image models (SDXL and Kandinsky 3). Prompts and generated images are provided in Appendix~\ref{appendix:usersurvey}. Users consistently favored DUNE-enhanced outputs, supporting the effectiveness of our approach in improving perceived quality and reducing hallucinations.

Finally, we conduct several additional experiments in Appendix~\ref{appendix:down}. We selectively scale skip connections and upsampling blocks both upwards and downwards during the detail phase, contrasting with previous methods that predominantly amplify these components~\cite{si2024freeu}. This enables precise control over brightness and tone, allowing both brighter and dimmer images.

\section{Conclusion}
\label{conclusion}

In this work, we analyzed the functional roles of U-Net components throughout the diffusion process, clarifying where and when artifact-associated internal dynamics appear during diffusion sampling. Based on these findings, we introduced DUNE, a training-free, phase-aware refinement framework whose core detect--suppress mechanism identifies and corrects anomalous deep latents, yielding consistent quality gains across both U-Net and Transformer backbones. The framework additionally offers an optional detail phase for user-driven stylistic adjustments. We extend DUNE into Transformer-based methods and also achieve performance improvements. Extensive evaluations show that DUNE reduces artifacts while preserving semantic consistency. 

\textbf{Limitations}
Although we suggest some guideline of choosing hyperparameters in sensitivity analysis, our method requires minimal hyperparameter tuning (percentile threshold and scaling factor). Fully automatic threshold setting remains for future work. 
\section*{Acknowledgement}

This work was supported by Institute of Information \& communications Technology Planning \& Evaluation (IITP) grant funded by the Korea government (MSIT) (No. RS-2022-II220984, Development of Artificial Intelligence Technology for Personalized Plug-and-Play Explanation and Verification of Explanation; No. RS-2024-00457882, AI Research Hub Project), and by the Korea Evaluation Institute of Industrial Technology (KEIT) grant funded by the Korea government (MOTIE) (No. RS-2025-25458052, Development of Core Technologies for Manufacturing Foundation Models).
\newpage
%
%
\bibliographystyle{splncs04}
\bibliography{main}

\clearpage
\setcounter{page}{1}

\section{Hyperparameter Setting and Miscellaneous environments}

\subsection{Sensitivity Analysis on Percentiles and Kappas}
\label{appendix:sensitivity}

\begin{figure}[t!]
\centering
\includegraphics[width=0.8\linewidth]{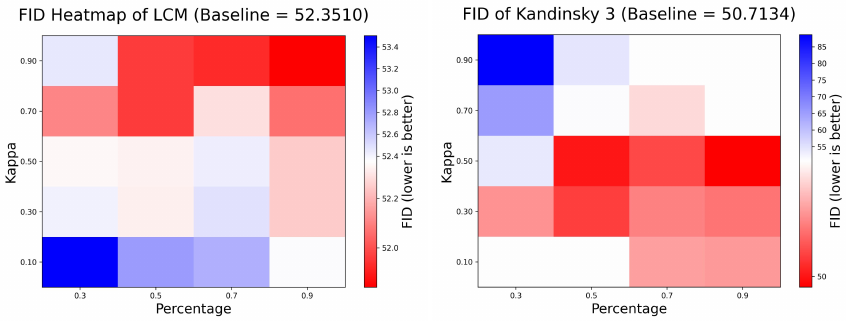}
\caption{Heatmap of FID across different percentages and kappas.}
\label{fig:ablation_pick}
\end{figure}

In this section, we investigate the sensitivity of DUNE with respect to its two hyperparameters: the detection percentile $p$ and the suppression scaling factor $\kappa$ (see Section~\ref{sec:analysis}). Specifically, we generate 1000 images using LCM and Kandinsky 3 models and compute their FID across combinations of hyperparameters $(p, \kappa) \in \{0.3, 0.5, 0.7, 0.9\} \times \{0.1, 0.3, 0.5, 0.7, 0.9\}$. 

Figure~\ref{fig:ablation_pick} illustrates heatmaps of FID values across these hyperparameter combinations. The best results consistently occur at $p=0.9$, corresponding to suppressing the top 10\% of clearly identified anomalies and thereby we set default value of $p$ by 0.9 (Table~\ref{tab:hyperparam}). FID scores generally improve (decrease) either with moderate suppression levels for Kandinsky 3 ($\kappa=0.5$) or at higher suppression levels for LCM.

\subsection{Implementation details}

\begin{table}[h]
  \centering
  \vspace{-1.5em}
  \caption{Detailed hyperparameter settings for DUNE.}
  \label{tab:hyperparam}
  \vspace{-0.5em}
  \resizebox{0.65\linewidth}{!}{%
    \renewcommand{\arraystretch}{1.4}
    \begin{tabular}{lccccc}
      \toprule
      & {SDXL} & {LCM} & {Kandinsky 3} & {PixArt-$\Sigma$} & {Hunyuan-DiT} \\
      \midrule
      \textbf{Kappa ($\kappa$)} & 0.3 & 0.3 & 0.5 & 0.95 & 0.85\\
      \bottomrule
    \end{tabular}
  }
  \vspace{-0.5em}
\end{table}

We provide the practical hyperparameter settings used throughout our experiments. The percentile threshold is fixed to $p=0.9$, as supported by the sensitivity analysis in Appendix~\ref{appendix:sensitivity}, and the EMA coefficient is fixed to $\gamma=0.7$ throughout all experiments. Consequently, the only backbone-dependent control parameter is the suppression factor $\kappa$. 
For Transformer backbones, we use a fixed target layer at $3/4$ depth of the
self-attention stack. This choice is motivated by the mid-to-deep denoising
property discussed in Section~\ref{sec:theorems}, qualitatively supported by Figure~\ref{fig:latent_vis}, and
quantitatively validated by the PixArt-$\Sigma$ ablation in
Appendix~\ref{appendix:transformer_layer_ablation}.
We also use a short warm-up before activating suppression to avoid intervening before stable semantics emerge; this is kept fixed as an implementation default rather than tuned per prompt or per image. Specifically, we start suppression after three steps for long samplers (SDXL and Kandinsky~3) and after one step for LCM. All remaining settings follow the default parameters of the underlying diffusion models.

All experiments are conducted by using the following software and hardware environments: \textsc{Ubuntu} 18.04 LTS, \textsc{Python} 3.12.3, \textsc{CUDA} 12.1, \textsc{NVIDIA} Driver 530.30.02, Intel Xeon Gold 6342 CPU, and \textsc{RTX A6000}.

\section{Empirical observation of the Transformer intervention}\label{appendix:attemperical}

This section supports the claim of Section~\ref{sec:theorems} that, in Transformer backbones, mid-to-deep self-attention latents become progressively denoised and therefore serve as the appropriate intervention space for DUNE. Our goal is twofold: (i) to visualize this depth-wise denoising trend, and (ii) to justify the practical choice of the target self-attention layer used in our experiments.

\subsection{Latent-map visualization across depth and time}

\begin{figure*}[t!]
\centering
\includegraphics[width=0.9\linewidth]{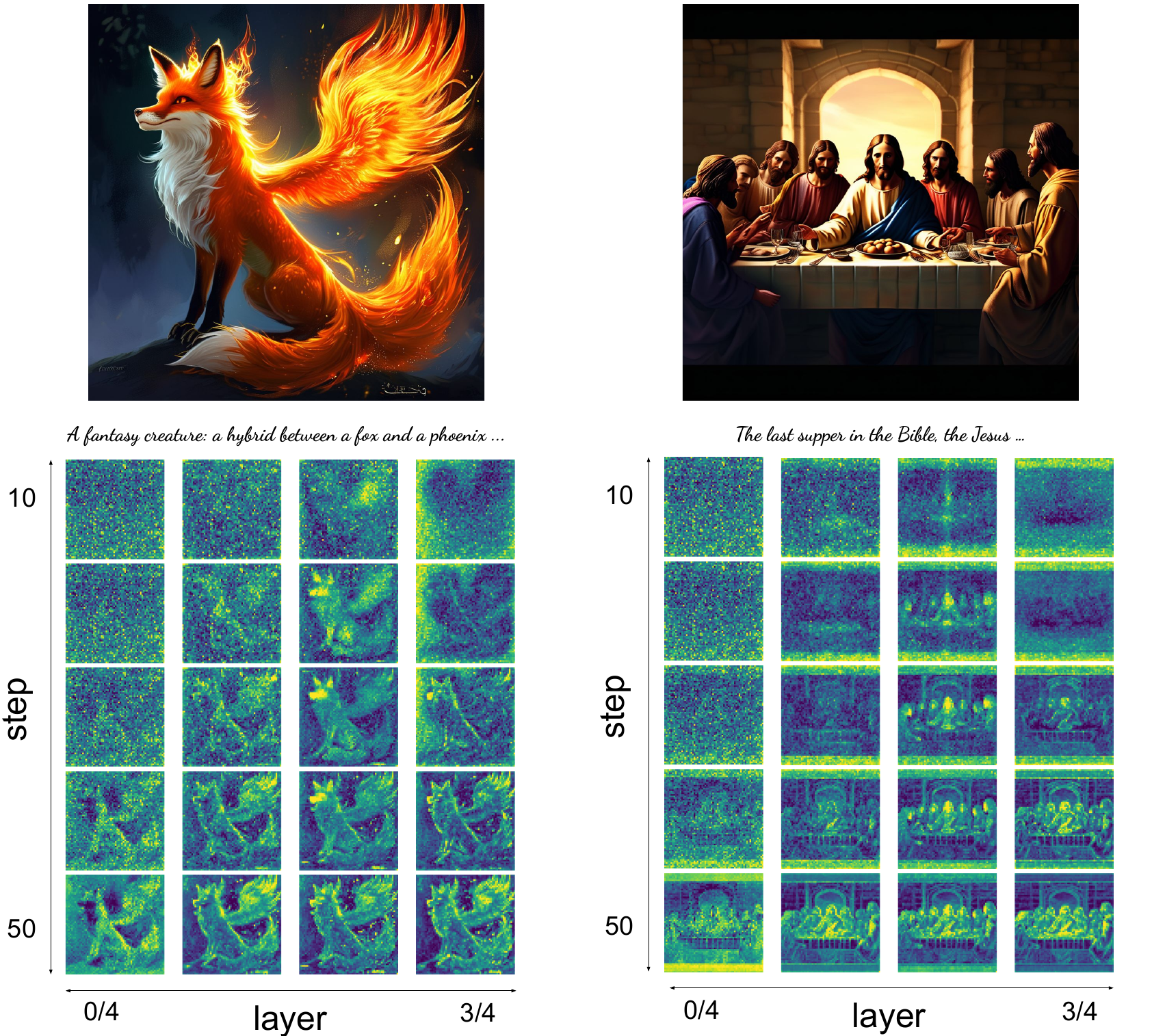}
\caption{Self‑attention latent maps over depth and time for PixArt‑$\Sigma$ and Hunyuan‑DiT. For each model, we visualize token maps obtained by channel‑averaging the post‑attention activations at four depths of the Transformer stack—$0/4$, $1/4$, $2/4$, and $3/4$ of the total attention depth—sampled every 10 denoising steps. The terminal layers are omitted from visualization because they primarily output the directional noise ($\epsilon$) following the diffusion model's objective.}
\label{fig:latent_vis}
\end{figure*}

In this section, we provide qualitative evidence supporting the theoretical results in Section~\ref{sec:theorems}. For U‑Net backbones, denoising concentrates in the downsampled bottleneck (``$h$‑space''). However, because this space has much lower spatial resolution than the input, purification is difficult to visualize directly and is typically inferred from theory and prior analyses~\cite{kwon2022diffusion}. 
In contrast, Transformer backbones operate on patchified tokens via self‑attention, which preserves spatial structure and enables direct visualization of latent maps.

Figure~\ref{fig:latent_vis} reports representative results for PixArt‑$\Sigma$ and Hunyuan‑DiT. 
At the first layer ($0/4$), the token maps are noisy, and as depth increases toward the middle ($1/4$–$2/4$), the maps become substantially cleaner. This trend matches the prediction of Proposition~\ref{thm:attn_softmax}: the softmax weights form a data‑adaptive kernel whose concentration factor $\sum_j w_j^2$ decreases in the middle layers, yielding stronger noise reduction. 
Near $3/4$ of the stack, the maps sharpen and become clearer. 
These representative examples are consistent with the proposed mid-to-deep denoising trend (here sampled every 10 denoising steps), supporting the use of mid‑to‑deep self‑attention representations as the Transformer analogue of the $h$‑space for detection and suppression.

\subsection{Target-layer ablation for Transformer backbones}
\label{appendix:transformer_layer_ablation}

\begin{table}[t]
\centering
\caption{Target-layer ablation on PixArt-$\Sigma$. We vary only the target
self-attention layer used by DUNE and keep all other settings fixed. Lower FID
is better.}
\label{tab:pixart_layer_ablation}
\setlength{\tabcolsep}{15pt}
\begin{tabular}{ccc}
\hline
Relative depth & Layer index & FID $\downarrow$ \\
\hline
$0/4$ & $1/28$  & 31.65 \\
$1/4$ & $7/28$  & 30.14\\
$1/2$ & $14/28$ & 28.95 \\
$3/4$ & $21/28$ & 27.80 \\
Last  & $28/28$ & 32.68 \\
\hline
\end{tabular}
\end{table}

The theoretical result in Section~\ref{sec:theorems} motivates applying DUNE to mid-to-deep
self-attention representations, where token features become semantically
integrated and increasingly denoised. This tendency is also visible in
Figure~\ref{fig:latent_vis}: early layers remain noisy, middle layers become cleaner, and features near $3/4$ depth appear sharper. To validate this design
quantitatively, we apply DUNE to PixArt-$\Sigma$ while varying only the target
self-attention layer. Specifically, we test four representative depths
($1/4$, $1/2$, $3/4$, and the last layer of the Transformer stack) while
keeping all other settings fixed. As shown in
Table~\ref{tab:pixart_layer_ablation}, the $3/4$-depth layer yields the best
FID, supporting our default choice for Transformer backbones. Overall, the
theory justifies a mid-to-deep operating region, while $3/4$ serves as a robust
empirical default rather than a claim of universal optimality.

Also, to clarify the scope of our unified claim, we do not assume that Transformer backbones admit the same branch-wise decomposition as U-Net skip connections
and upsampling paths. Instead, we claim that both architectures contain a deep
internal latent in which semantic structure is more integrated and stochastic
fluctuations are relatively attenuated, making it a stable intervention point
for DUNE.

Figure~\ref{fig:latent_vis} qualitatively supports this view: for both PixArt-$\Sigma$ and
Hunyuan-DiT, post-attention token maps become progressively cleaner from early
to middle layers, while deeper layers remain semantically sharper. We further
validate the practical target-layer choice on PixArt-$\Sigma$ by applying DUNE
at four representative depths ($1/4$, $1/2$, $3/4$, and the last layer) while
keeping all other settings fixed. As shown in
Table~\ref{tab:pixart_layer_ablation}, the $3/4$-depth layer achieves the best
FID, supporting our default choice for Transformer backbones. Thus, the unified
aspect of DUNE lies in the detect--suppress principle and the choice of a deep
semantically integrated intervention space, rather than in enforcing identical
component-level analyses across U-Net and Transformer architectures.

\section{Robustness on non-hallucinated images}\label{appendix:ablation}

\begin{figure}[t!]
    \centering
    \begin{subfigure}[b]{0.8\linewidth}
        \centering
        \includegraphics[width=\linewidth]{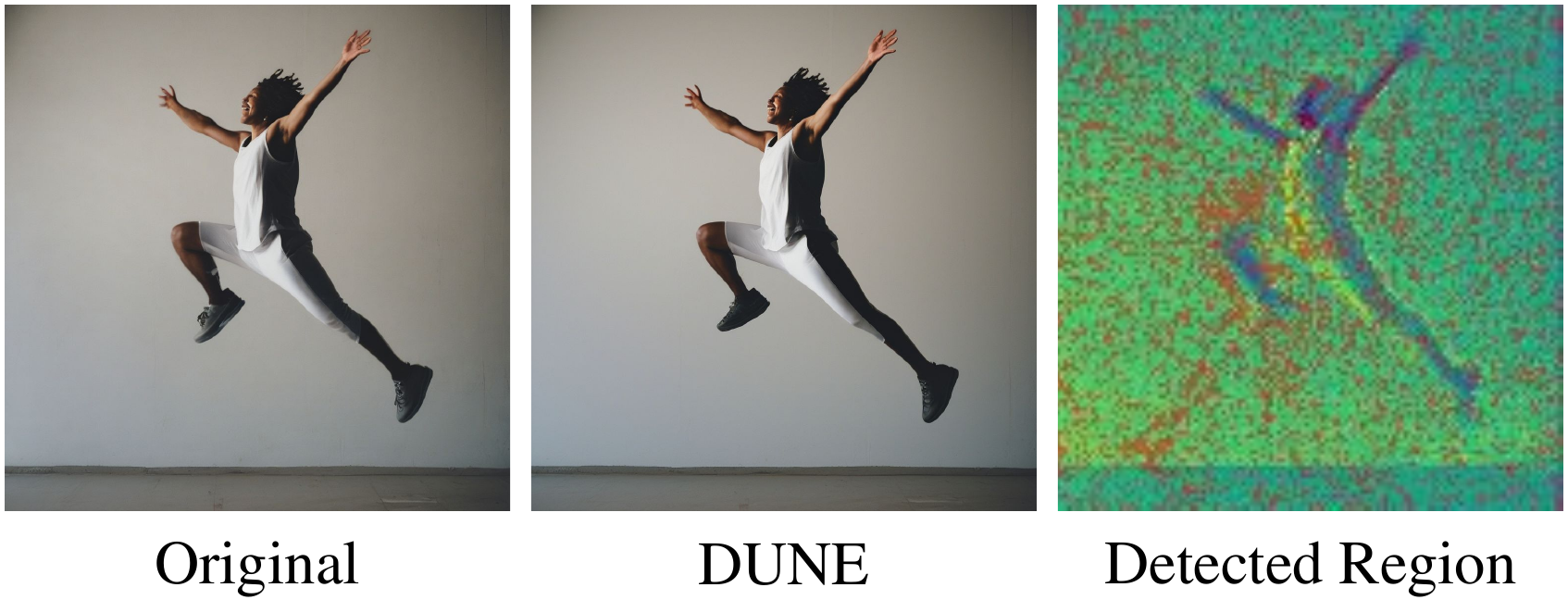}
        \caption{Visualization of non-artifact scenarios. This image was generated by SDXL with the prompt ``\textit{A person jumping with arms raised}.''}
        \label{fig:ablation_nonhall}
    \end{subfigure}

    \vspace{1em} 

    \begin{subfigure}[b]{0.8\linewidth}
        \centering
        \includegraphics[width=\linewidth]{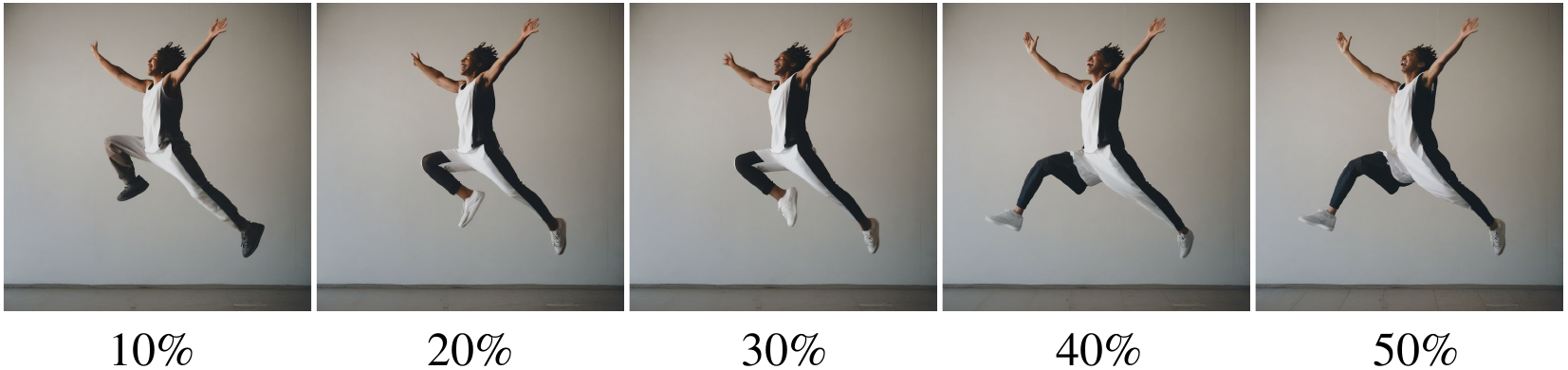}
        \caption{Results obtained by randomly suppressing h-space features.}
        \label{fig:ablation_random_perc}
    \end{subfigure}

    \caption{Investigation on applying DUNE to non-hallucinated images.}
\end{figure}

\begin{table}[t]
\centering
\caption{Random-mask suppression ablation on SDXL. We randomly select h-space entries during the detect phase and apply the same suppression operator, while keeping all other settings fixed.}
\label{tab:random_mask_ablation}
\setlength{\tabcolsep}{6pt}
\begin{tabular}{ccccccc}
\hline
Mask ratio & Original & 10\% & 20\% & 30\% & 40\% & 50\%\\
\hline
FID $\downarrow$ & 18.86 & 18.92 & 19.11 & 19.51 & 20.14 & 20.59 \\
\hline
\end{tabular}
\end{table}

For images that exhibit no clear visual artifacts or anatomical distortions, it might seem irrational to consistently correct a fixed percentile of the h-space. To investigate this concern, we first visualize how DUNE identifies potentially problematic areas in the latent space for non-artifact images. As demonstrated in Figure~\ref{fig:ablation_nonhall}, when no artifacts are present, the detected areas (marked in red dots) do not concentrate on any specific region, but instead spread evenly throughout the latent space.

Therefore, to further test the robustness of DUNE, we applied our suppression method randomly across the h-space, using the same correction scaling factor ($\kappa$) as in our main experiments. Interestingly, Figure~\ref{fig:ablation_random_perc} illustrates that image semantics remain largely intact even when up to 50\% of the latent space features are randomly suppressed. Specifically, when only 10\% of features are suppressed randomly, the visual quality remains nearly indistinguishable from the original, emphasizing the stability and semantic preservation capabilities of DUNE.

Table~\ref{tab:random_mask_ablation} further highlights the importance of informative masking. 
When we replace the detector-identified mask with a random mask and apply the same h-space suppression operator in SDXL, FID degrades consistently as the masking ratio increases. 
This suggests that the improvement of DUNE does not come from arbitrarily suppressing early h-space features, but from selectively suppressing detector-identified anomalous regions. 
At the same time, the relatively modest degradation under random masking is consistent with Figure~\ref{fig:ablation_random_perc}, indicating that the suppression operator itself is not overly destructive to non-hallucinated samples.

\section{Metrics}\label{appendix:metric}

\begin{figure}[t!]
\centering
\includegraphics[width=0.95\linewidth]{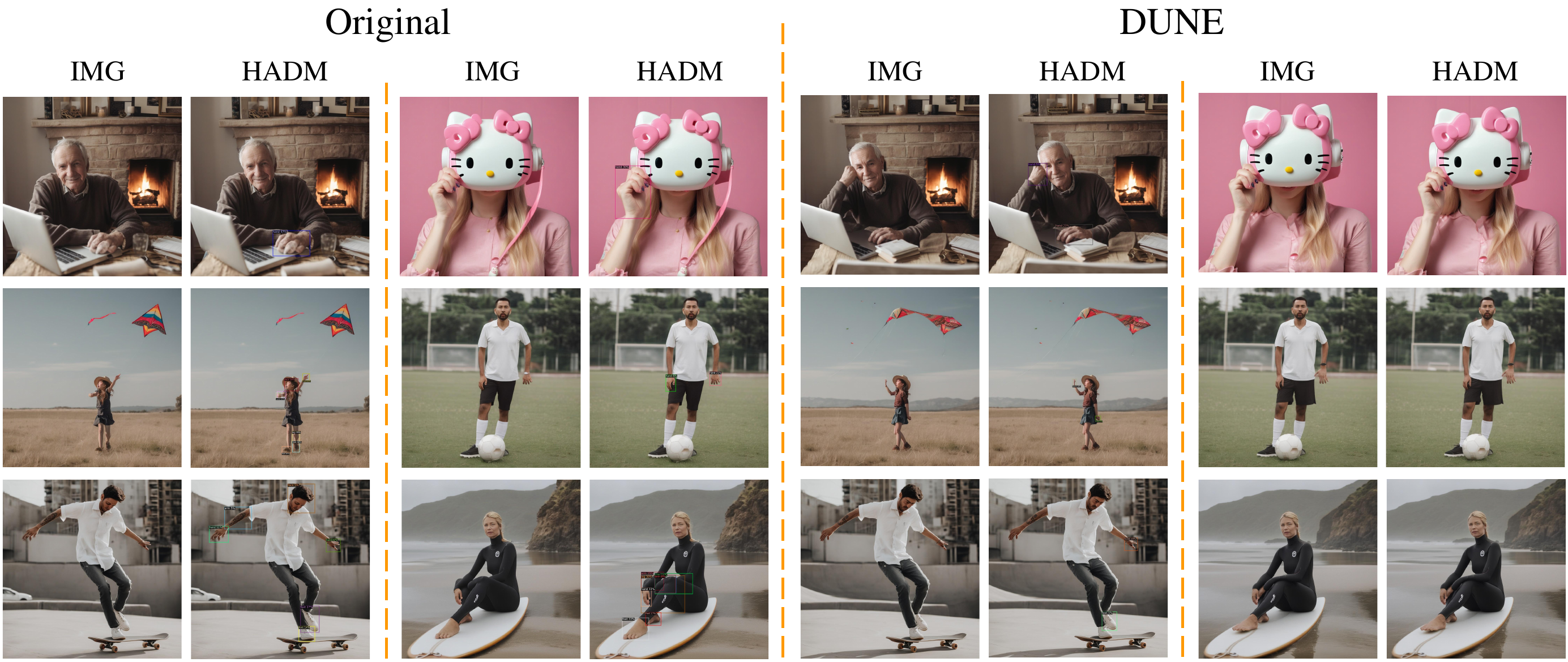}
\caption{Representative examples of detected samples in SDXL through HADM.}
\label{fig:hadm}
\end{figure}

Also, we use HADM-L~\cite{HADM} to measure the efficacy of DUNE in mitigating localized anatomical distortions. By employing a ViTDet-based object detection architecture, HADM-L explicitly localizes and classifies fine-grained structural defects. Specifically, it predicts bounding boxes around implausible body parts (e.g., six-fingered hands, distorted facial features, or unnaturally bent joints), while maintaining a low false-positive rate by leveraging normal human images during training, as shown in Figure~\ref{fig:hadm}

We deliberately report CLIP similarity (cosine similarity between 
original and refined images) rather than CLIP score (text--image 
alignment) for the following reason: we observed that methods 
producing overly saturated or high-contrast images---which deviate 
from the intended prompt semantics---can paradoxically achieve 
higher CLIP scores, as the CLIP encoder tends to favor visually 
salient features regardless of actual prompt faithfulness. This 
makes CLIP score unreliable as a quality indicator in our 
refinement setting. CLIP similarity, by contrast, directly 
measures whether the refinement preserves the semantic content 
of the original generation, which is the relevant criterion for 
a training-free post-hoc method.

\section{Comparison with Training-Free Refinement Baselines}\label{appendix:refinement}

\begin{figure}[t!]
\centering
\includegraphics[width=\linewidth]{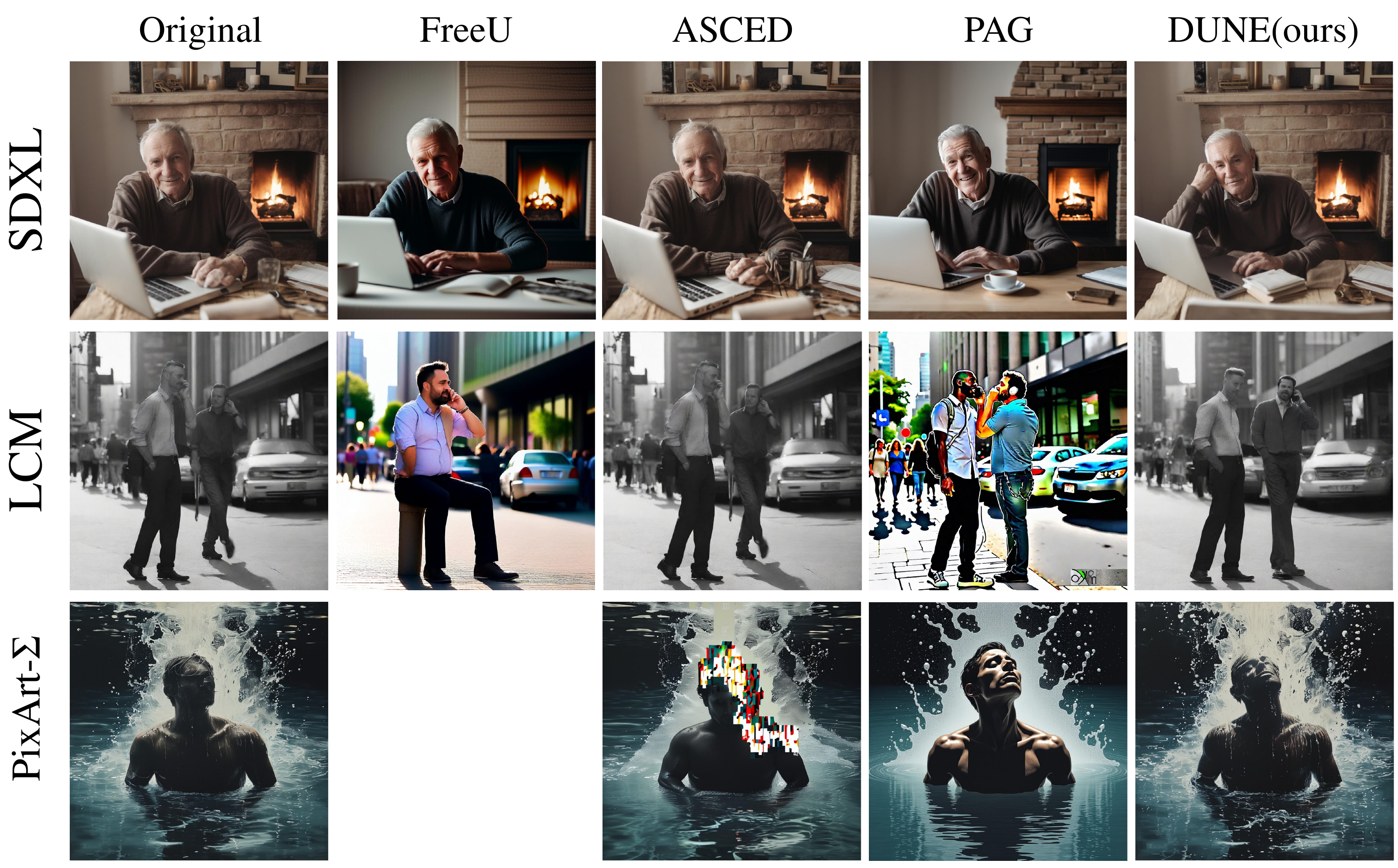}
\caption{Visual comparison between latest unsupervised refinement methods.}
\label{fig:refinemethods}
\end{figure}

\subsection{Broad Comparison}\label{appendix:refinement}

We consider three public, \emph{training‑free} refinement families:

\begin{itemize}
    \item \textbf{FreeU}\cite{si2024freeu} – reduces skip‑features and amplifies upsampling‑features. It is the most direct predecessor that DUNE aims to improve upon.
    \item \textbf{ASCED}\cite{cao2025temporal} – masks score‑trajectory outliers in pixel space and replaces them, standing for \textit{score‑masking} anomaly suppression.
    \item \textbf{PAG}\cite{PAG} – perturbs selected attention pathways at test time to guide denoising away from undesired structures, standing for \textit{attention-space guidance} approaches.
\end{itemize}

These three methods cover the most prominent families of training‑free
refinement: component scaling, second‑pass denoising, and
trajectory‑based masking, providing a balanced yard‑stick for DUNE. To maximize reproducibility and avoid cherry-picking, we follow the authors’ official implementations and recommended/default hyperparameters for FreeU, PAG, and ASCED.

The table~\ref{tab:metric} shows that DUNE achieves the best score on every metric‑model pair. For instance on SDXL, DUNE improves FID over Original by 0.11 and outperforms other models while avoiding the severe degradation introduced by FreeU (+12.4). HADM largely drops, indicating fewer anatomical distortions, whereas ASCED—despite a slightly better FID—raises HADM, suggesting that its pixel‑space masking can break body consistency.

\begin{figure}[h!]
\centering
\includegraphics[width=\linewidth]{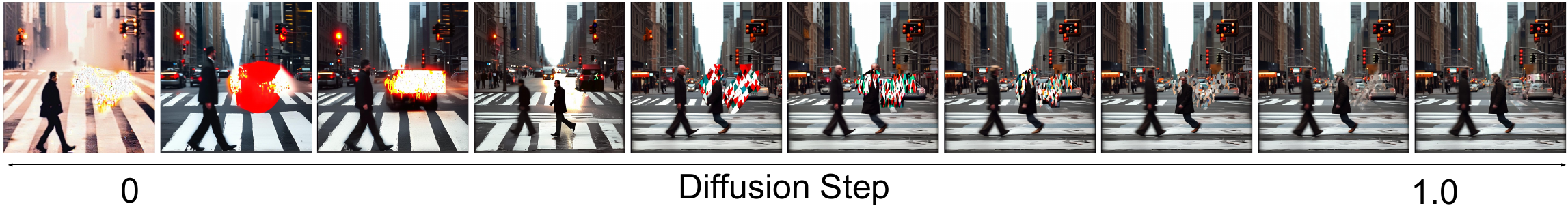}
\caption{Sampling results of different correction interventions across various time steps. }
\label{fig:asced}
\end{figure}

As depicted in Figure~\ref{fig:refinemethods}, we observe that PAG frequently generates over-saturated images (as mentioned in \cite{PAG}), a flaw it shares with FreeU. More importantly, this over-saturation problem is exacerbated when the number of sampling steps decreases. For example, on SDXL (50 steps), the saturation remains comparable to the original image; however, on PixArt-$\Sigma$ (30 steps), the over-saturation becomes visually intrusive. The degradation reaches its peak on LCM (8 steps), where extreme color distortion severely worsens the FID score (cf. Table~\ref{tab:metric}). In contrast, DUNE maintains a natural and moderate saturation level consistently across both short and long sampling steps.

Although PAG generates more clearer images, it loses detailed background or complex pattern. Moreover, such over-saturation becomes serious on less sampling steps; on SDXL (sampling step is 50) the saturation of image is similar with original image but on PixArt-$\Sigma$ (Sampling step is 30), saturation is much more obvious and on LCM (sampling step is 8), saturation is extremely high, making FID seriously worse (cf. Table~\ref{tab:metric}. In contrast, DUNE sustains moderate saturation across short sampling steps to longer one.

Also, we observe that ASCED often generates severe artifacts when applied to Transformer-based architectures. Because ASCED perturbs the sample and takes a forward step to recover the diffusion coefficient, it forces the reverse process to deviate from its original trajectory. While U-Net backbones can accommodate such deviations, Transformer-based models are notoriously sensitive to diffusion path alterations. Furthermore, this trajectory deviation introduces frequent \emph{semantic drift}; since ASCED applies pixel-space masking within the latent space equally across all channels, global structures (e.g., pose, background layout) can shift noticeably.

To isolate this issue, we re-forwarded the diffused sample without adding perturbed noise to examine the exact side effects of trajectory deviation on the Transformer-based method, following ASCED protocol. As illustrated in Figure~\ref{fig:asced} on PixArt-$\Sigma$, this deviation incurs severe noise even in the final phases of generation. Finally, ASCED requires almost \textit{double} the wall-clock time compared to standard sampling, as it mandates a full forward trajectory to compute score accelerations followed by a second denoising run with state replacement. In contrast, DUNE avoids both the destructive trajectory deviations and computational overhead by adding only a few tensor-wise masking operations inside a single pass.

\subsection{Why DUNE over FreeU?}\label{appendix:freeu}

\begin{figure}[t!]
\centering
\includegraphics[width=0.8\linewidth]{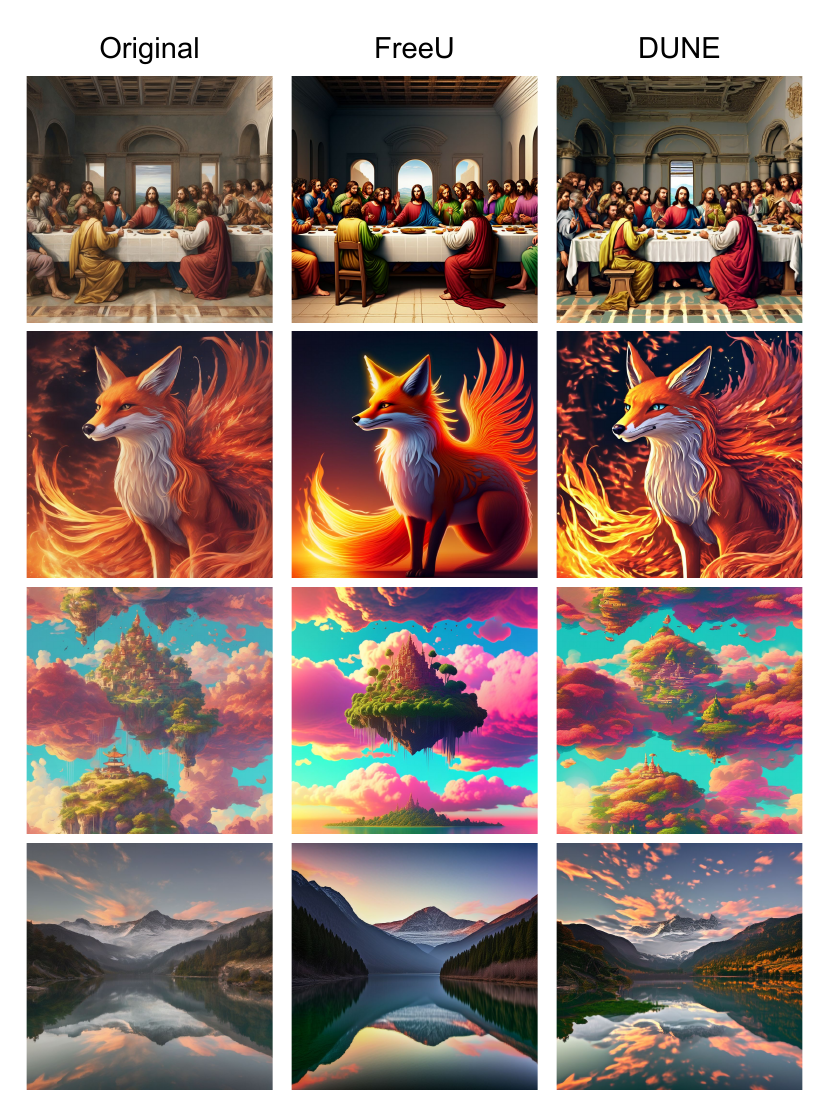}
\caption{Qualitative comparison between FreeU and DUNE (ours).}
\label{fig:freeu}
\end{figure}

In this section, we provide a detailed comparison with FreeU~\cite{si2024freeu}, another training-free tuning method designed for enhancing diffusion model outputs. 
FreeU is the closest baseline to DUNE because both methods manipulate internal features of pretrained diffusion backbones without retraining. The key difference is phase-aware control: DUNE first stabilizes anomalous deep latents during the detect phase and applies optional component reweighting only later, whereas FreeU reweights internal features without an explicit anomaly-correction phase. This distinction explains why DUNE better preserves context while avoiding the oversaturation and texture simplification often observed in FreeU.

As Figure~\ref{fig:freeu} shows, while FreeU tends to generate images characterized by overly saturated colors, monotone structures, and excessively smooth textures, DUNE produces outputs that are more diverse and visually coherent. Specifically, FreeU's adjustments often lead to loss of nuanced details and reduced visual diversity. In contrast, DUNE introduces an explicit initial correction phase targeting anomalies in the h-space, followed by adaptive scaling of skip connections and upsampling blocks based on a detailed, phase-wise and depth-wise analysis of the U-Net architecture. This approach enables DUNE to enhance fine details, preserve semantic consistency, and adaptively control visual attributes such as brightness and saturation, offering users enhanced flexibility in image refinement (as detailed in Sections~\ref{sec:analysis},~\ref{sec:sup-analysis}, and Appendix~\ref{appendix:theoretical_justification}).

Table~\ref{tab:metric} shows that FreeU~\footnote{We use suggested hyperparameters in official code (https://github.com/ChenyangSi/FreeU).} significantly deteriorates FID scores, indicating a loss of semantic consistency relative to the original images. In contrast, DUNE effectively improves image quality while preserving semantic fidelity, thus combining the advantages of FreeU without compromising semantic integrity. Notably, we emphasize that FreeU exclusively focuses on reducing skip connections to remove noise and increasing upsampling blocks to highlight original semantics. Ironically, these modifications often eliminate essential details from the original images and result in overly saturated outputs. Conversely, DUNE maintains original semantic content while enabling users to flexibly control saturation and brightness according to their preferences.

\section{Investigation on Remaining Component}\label{appendix:down}

\begin{figure}[h]
    \centering
    \begin{subfigure}[b]{0.7\linewidth}
        \centering
        \includegraphics[width=\linewidth,keepaspectratio]{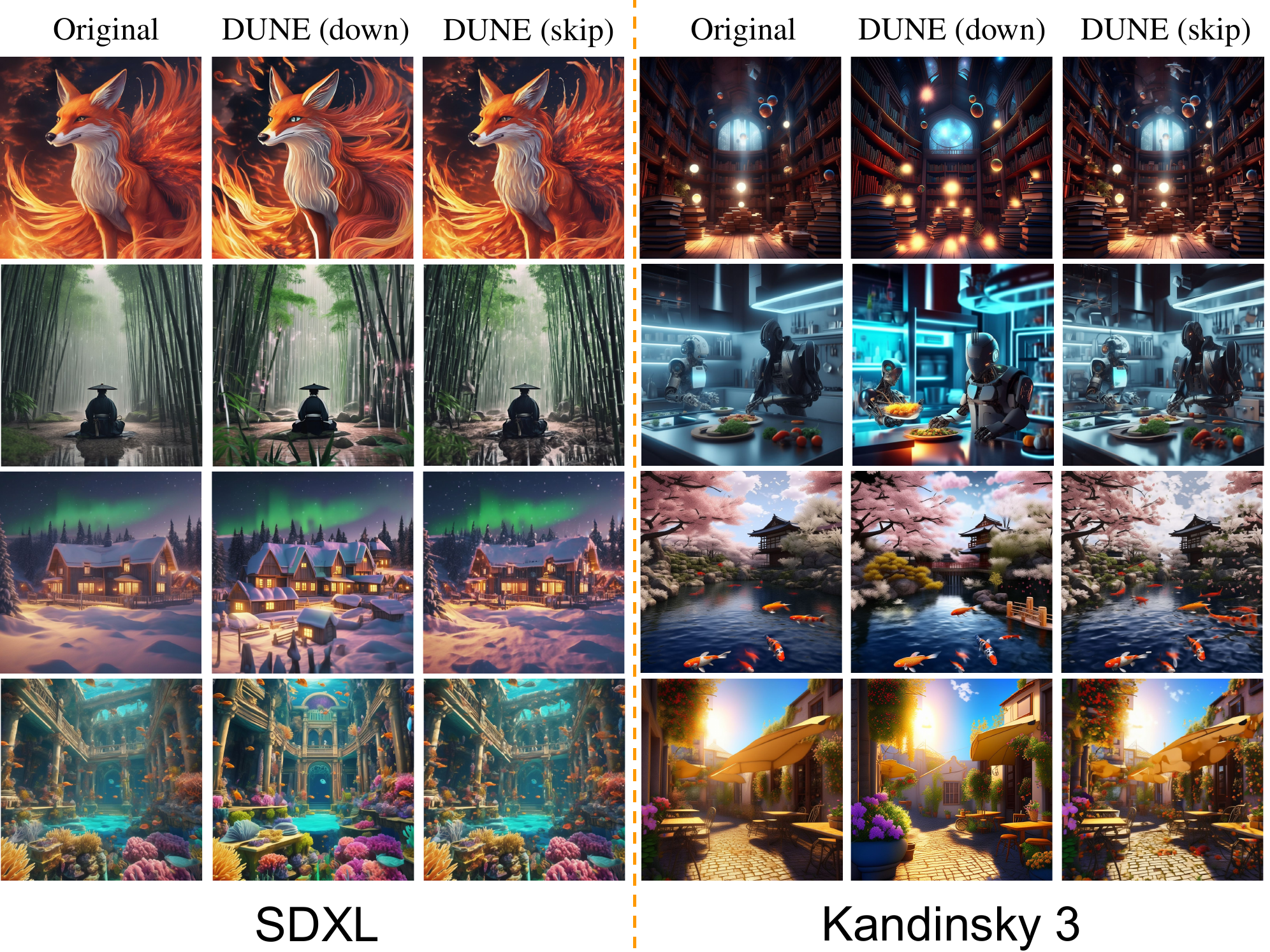}
        \caption{Results of controlling downsampling blocks and skip connections.}
        \label{fig:diversity2}
    \end{subfigure}

    \vspace{0.5em}

    \begin{subfigure}[b]{0.7\linewidth}
        \centering
        \includegraphics[width=\linewidth,keepaspectratio]{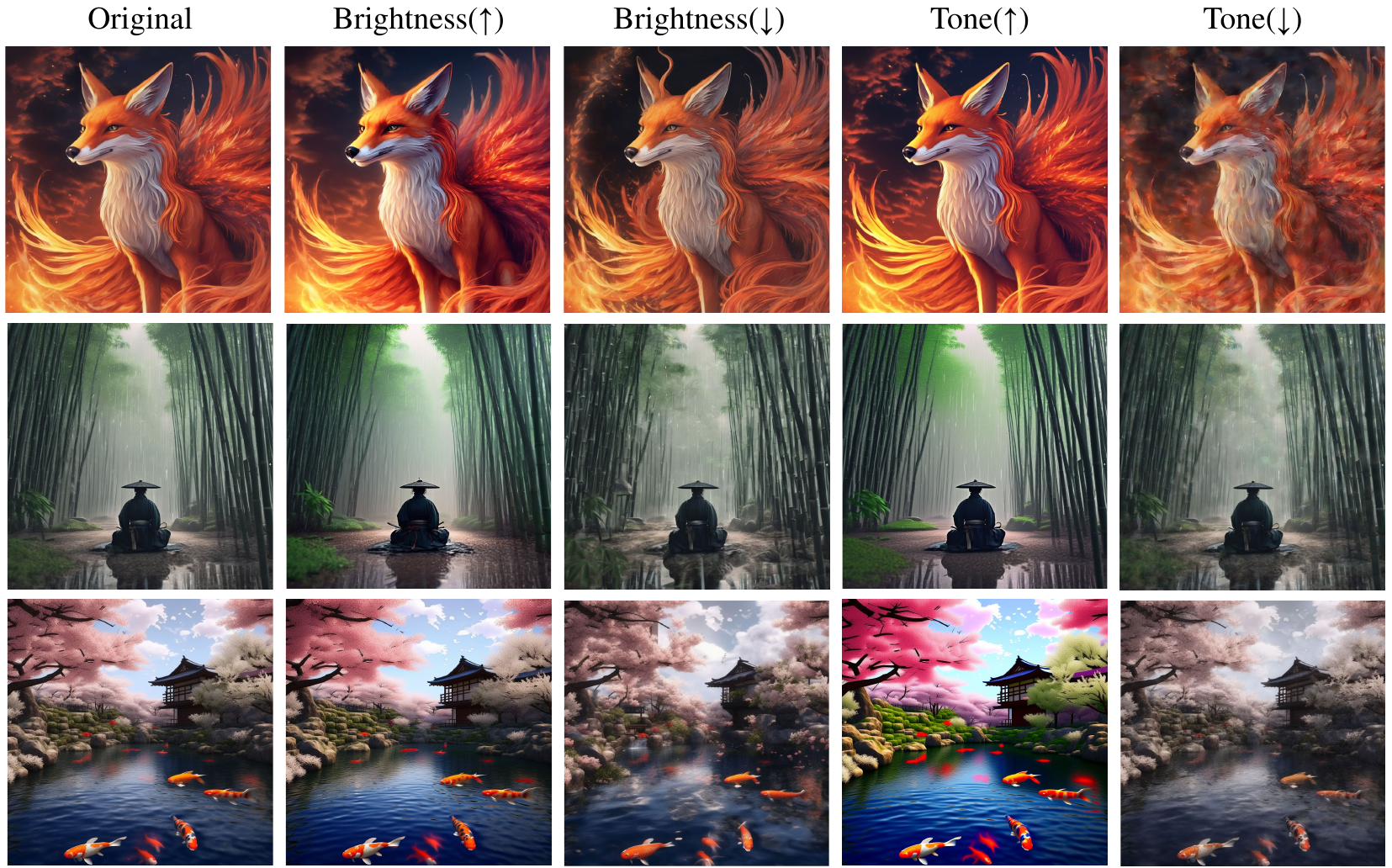}
        \caption{Ablation study on upsampling blocks.}
        \label{fig:ablation_control_unet}
    \end{subfigure}

    \caption{Visualization of controlling different components of the U-Net architecture.}
    \label{fig:combined_control_unet}
\end{figure}

Building upon previous observations regarding skip connections, h-space, and upsampling blocks, we further investigate the role of downsampling blocks. Earlier, we identified that skip connections primarily carry high-frequency information (e.g., edges), while upsampling blocks handle lower-frequency content (e.g., overall color or contents). Intuitively, downsampling blocks integrate comprehensive information by combining multiple components. Figure~\ref{fig:diversity2} illustrates that amplifying downsampling blocks, in the same way of scaling upsampling blocks, introduces additional detail but risks deviating from the original semantic content.

\begin{figure}[ht!]
\centering
\includegraphics[width=0.6\linewidth]{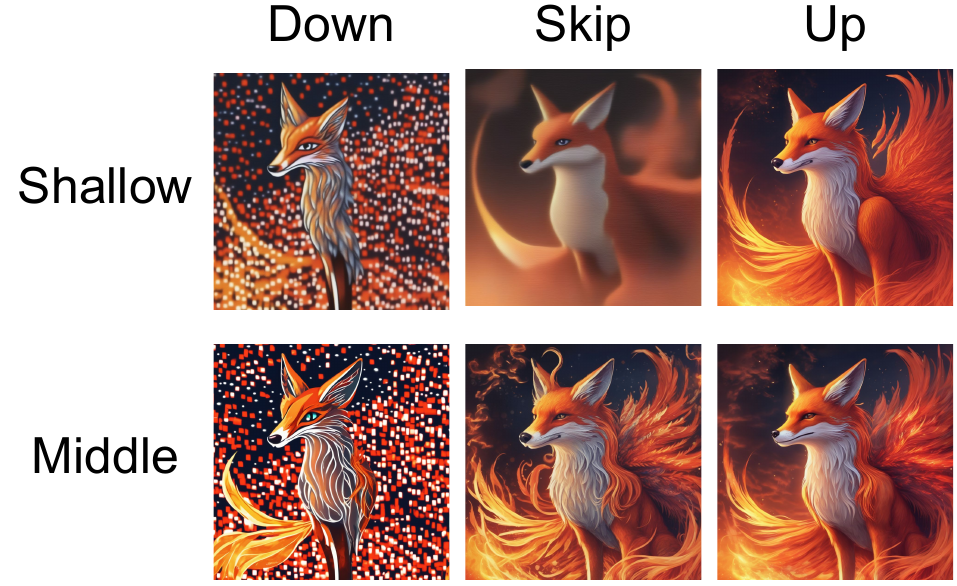}
\caption{Ablation study depth-wise control of each component.}
\label{fig:ablation_style}
\end{figure}

We additionally perform a depth-wise analysis of downsampling blocks. As depicted in Figure~\ref{fig:ablation_style}, controlling shallower layers is generally not recommended, aligning with the theoretical justification for preferring deeper layers as provided in Appendix~\ref{appendix:theoretical_justification}. Notably, effective control of downsampling blocks is primarily beneficial at the deepest layers, unlike other components, for which mid-layer adjustments typically remain visually acceptable. This aligns with the intuition that manipulating downsampling blocks simultaneously influences skip connections and upsampling blocks, thus requiring more careful handling.

Moreover, we investigate alternative scaling approaches for skip connections and upsampling blocks. Contrary to FreeU, which conventionally reduces skip connections and increases upsampling blocks, we explore the opposite scenario. As Figure~\ref{fig:diversity2} shows, amplifying skip connections enhances image detail while preserving semantic integrity. Figure~\ref{fig:ablation_style} further demonstrates that reducing upsampling blocks effectively decreases brightness and tones, allowing users to tailor the visual characteristics of generated images to their preferences. Specifically, we observe that initiating upsampling block adjustments slightly earlier in the diffusion process (approximately between 60\% and 70\%) primarily affects image brightness, whereas later adjustments predominantly control tonal qualities.

\section{Related Work}
\label{appendix:relatedworks}
\textbf{Diffusion models.} Diffusion models are a class of generative model, which have recently showed remarkable performance. They consist of a forward process---which progressively adds Gaussian noise to clean samples until they transform into Gaussian noise---and a reverse process---which generates real samples by sequentially denoising noisy samples.
Specifically, given a clean sample $X_0 \sim q(X_0)$, where $q$ is the data distribution, the forward process yields a noisy sample $X_t$ at time step $t$, which can be described as  

\begin{equation} \label{eq:x_t}
    X_t=\sqrt{\bar{\alpha}}_tX_0+\sqrt{1-\bar{\alpha}_t}\epsilon \quad \epsilon \sim \mathcal{N}(0, \textbf{I})
\end{equation}
where $\bar{\alpha}_t$ is a noise schedule, adopting the notation of DDPM~\cite{ho2020denoising}.

Diffusion models aim to learn the reverse process that can generate real samples from Gaussian noise. To achieve this, most prior works employed the loss function designed to predict the injected noise $\epsilon$~\cite{ho2020denoising} or, equivalently, the score function~\cite{song2019generative} of the noisy data distribution 
\begin{equation}\label{eq:score_loss}
    \mathcal{L}(\theta)=\mathbb{E}_{X_{t}, t}[\Vert \epsilon_{\theta}(X_{t}, t) - \epsilon\Vert _2^2]
\end{equation}

\textbf{U-net architecture.} U-net~\cite{ronneberger2015u} is the most commonly adopted architecture for approximating the \textit{score function} in diffusion models. As shown in Fig.~\ref{fig:framework}, U-net is mainly composed of three components: \textit{downsampling blocks}, \textit{skip connections} and \textit{upsampling blocks}. At each depth of U-net, a downsampling block passes its output to the deeper block while also conveying it through a skip connection to the upsampling block at the same depth. 

Although most of the diffusion models employ U-net architecture, only a limited number of studies have explored the components of U-net architecture. \cite{kwon2022diffusion} focused on the bottleneck of U-net, which is the deepest block, revealing that the bottleneck learns disentangled semantic space (h-space) and semantic manipulation is available by controlling the h-space of U-net. \cite{li2023faster} compared downsampling and upsampling blocks, showing that downsampling feature maps change more slowly across the denoising timesteps compared to those of upsampling blocks. Furthermore, by presenting the norm of encoder and decoder feature maps, they argued that downsmapling blocks play less important role than upsampling blocks. On the other hand, FreeU~\cite{si2024freeu} investigates the role of and upsampling blocks and skip connections. They argued that upsampling blocks mainly perform denoising and skip connections convey high-frequency information to upsampling blocks.  

However, none of the previous works have focused on the role of each block jointly with the denoising time steps, which is crucial to understand the detailed mechanism of the denoising process. Thus, in this work, we provide a thorough investigation into the role of U-net blocks across the denoising time steps and introduce a simple and effective refinement method based on the understandings. 

\textbf{Denoising process.}
\label{related:denoising}
Previous studies~\cite{choi2022perception, yang2023diffusion} analyzed the denoising process of diffusion models, revealing that low-frequency components are recovered in the early denoising steps, while high-frequency details are restored in the later steps. This behavior arises from the forward process, where high-frequency details destroy in the early phase, and low-frequency content gradually fades away. However, these works do not investigate the role of U-net components across denoising steps, which is the focus of our study. 

\textbf{Hallucination in diffusion models.}
Some studies~\cite{aithal2024understanding, cao2025temporal} have investigated hallucination of diffusion models. A recent work~\cite{aithal2024understanding} focuses on mode interpolation which results in the generation of unrealistic samples. They observe that diffusion models tend to interpolate between nearby modes, and their analysis reveals that the smooth approximation of the score function leads to such interpolation. Furthermore, they demonstrate that mode-interpolated samples exhibit high variance during the denoising trajectories. Another study~\cite{cao2025temporal} examines the temporal dynamics of the learned score function. They identify that the score dynamics of artifact regions show abnormally large variations compared to normal regions during the denoising process. 

\textbf{Related Training-Free Refinement Methods}

Training-free refinement methods differ in \emph{where} they intervene.
Score-level approaches---TAG~\cite{cho2025tag} (tangential amplification of sampling increments),
Dynamic Guidance~\cite{triaridis2025dynamic} (selective score sharpening),
and ASCED~\cite{cao2025temporal} (corrective noise injection)---modify the output or trajectory without explicitly accessing internal features, and therefore do not directly target detector-selected internal latent deviations.
We retain ASCED as a baseline for its shared artifact-suppression goal;
TAG and Dynamic Guidance operate at a complementary pipeline stage
(post-score), making direct comparison less informative.

At the attention level, PAG~\cite{PAG} perturbs self-attention
identity as guidance, and AAM~\cite{ozbulak2025aam} applies softmax
temperature scaling to suppress hallucinations.
PAG is included as a baseline; AAM is excluded as it only supports
low-resolution unconditional DDPM (MNIST, Hands).

For internal-latent manipulation, FreeU~\cite{si2024freeu} reweights
backbone and skip features;
InjectFusion~\cite{jeong2024injectfusion} blends $h$-space features
for content injection, validating $h$-space as an intervention point.
FreeU is included as a baseline for its shared reweighting paradigm;
InjectFusion and SkipInject target editing rather than artifact
suppression, precluding meaningful comparison.

\section{Evaluation on Unconditional Generation}

\begin{figure}[t]
  \centering
  \begin{subfigure}[c]{0.55\textwidth}
    \centering
    \includegraphics[width=\linewidth]{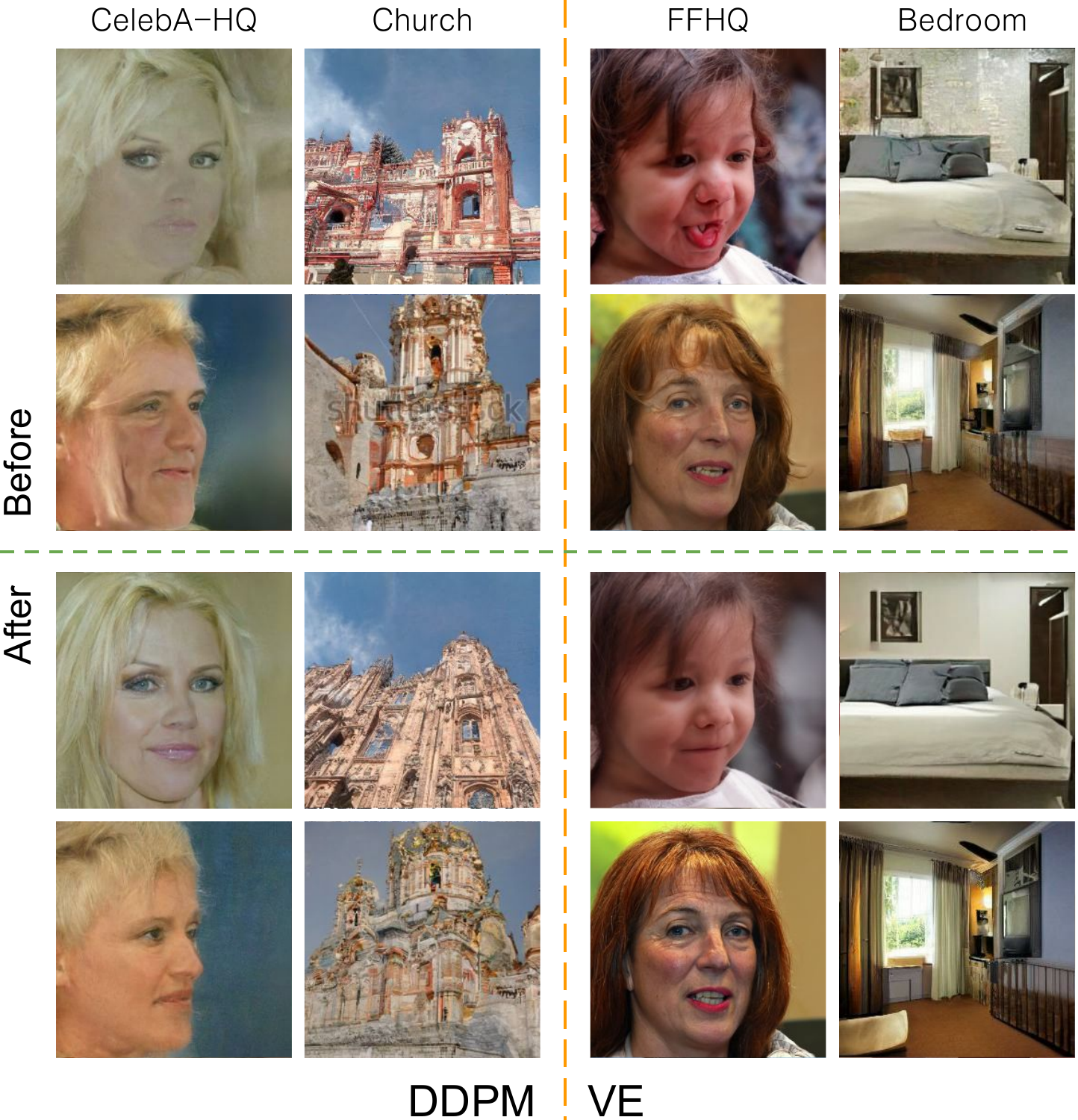}
    \caption*{(a) Visual results on low-resolution image generation.}
  \end{subfigure}%
  \hfill
  \begin{subfigure}[c]{0.42\textwidth}
    \centering
    \caption*{\textbf{(b) Qualitative Results}: User survey from 50 volunteers.}
    \vspace{0.5em}
    \resizebox{\linewidth}{!}{%
      \renewcommand{\arraystretch}{1.4}
      \begin{tabular}{cccc}
        \toprule
        \textbf{Model} & \textbf{Dataset} & \textbf{Original} & \textbf{DUNE} \\
        \midrule
        \multirow{2}{*}{DDPM} & CelebA-HQ & 6.60\% & \bfseries 93.40\% \\
                              & Bedroom   & 15.28\% & \bfseries 84.72\% \\
        \midrule
        \multirow{2}{*}{VE}   & FFHQ      & 11.51\% & \bfseries 88.49\% \\
                              & Church    & 16.67\% & \bfseries 83.33\% \\
        \bottomrule
      \end{tabular}
    }
  \end{subfigure}

  \caption{Comparison between visual results and user preference scores in unconditional generation tasks.}
  \label{fig:eval_uncond_swapped}
\end{figure}


While our main qualitative analyses focused on models such as SDXL and Kandinsky 3, we also conducted an additional user survey to evaluate DUNE's performance on DDPM and VE models. Due to the unavailability of the original training datasets for these models, we omitted their formal quantitative evaluations in the main paper. Nonetheless, qualitative improvements are strongly supported by user preference data. As summarized in Figure~\ref{fig:eval_uncond_swapped}(b), a substantial majority of the 50 surveyed participants consistently preferred the images generated by DUNE across all tested scenarios. Specifically, DUNE-enhanced images significantly outperformed original images, achieving user preference scores of 93.40\% on CelebA-HQ and 88.49\% on FFHQ. These results clearly demonstrate DUNE's effectiveness in enhancing unconditional image generation, underscoring strong user preference and broad applicability across various datasets and diffusion model architectures.

\section{Remained result of Figure~\ref{fig:main}}\label{appendix:remained_figures}

\begin{figure}[t]
\centering
\includegraphics[width=0.7\linewidth]{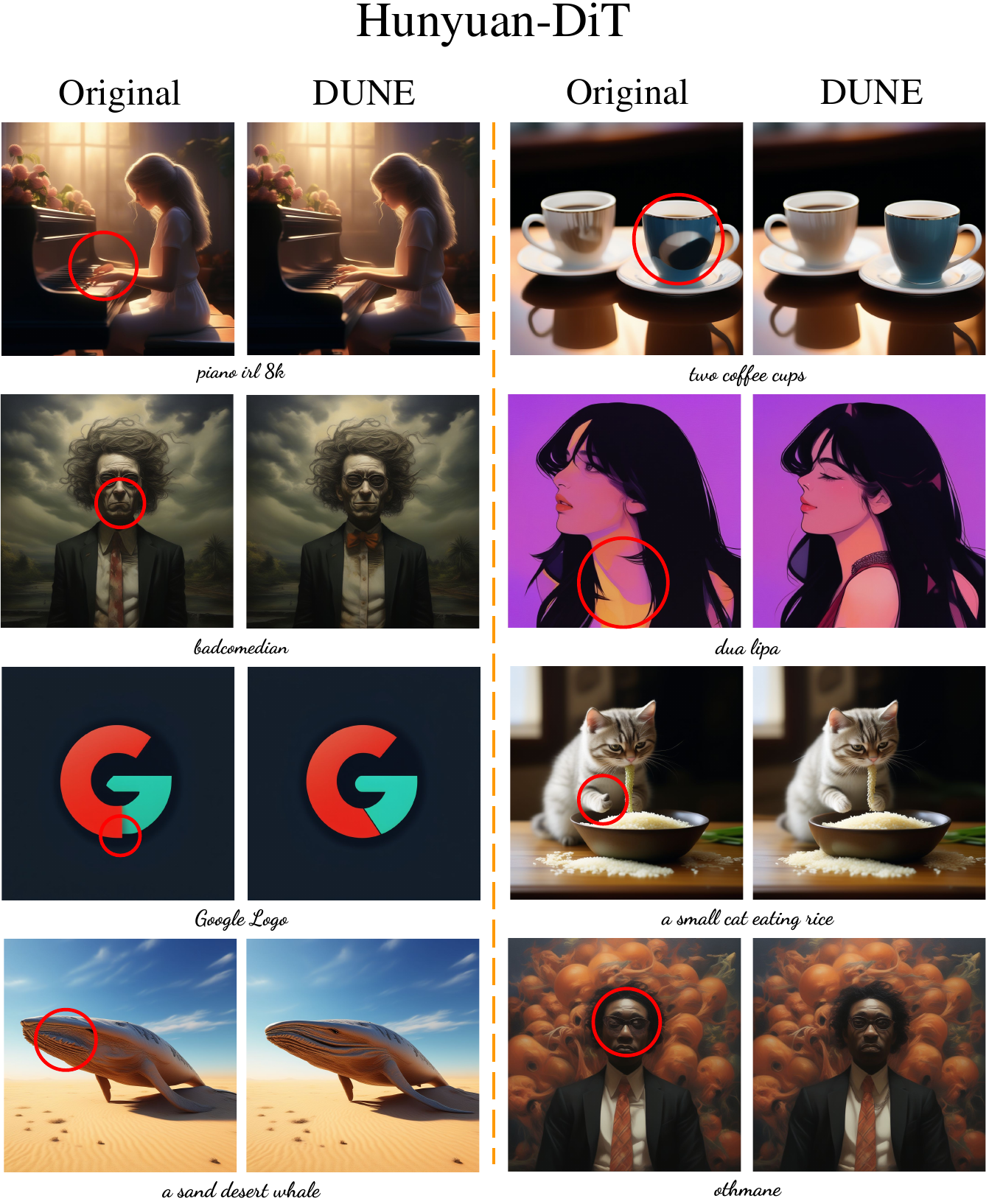}
\caption{Qualitative analysis on Hunyuan-DiT.}
\label{fig:Hunyuan}
\end{figure}

In this section, we present experimental results for another high-resolution T2I model—Hunyuan-DiT—which was omitted from Figure~\ref{fig:main} due to space limitations. The figure shows that DUNE improves the visual quality of the generations.

\section{Proofs of Theorems}

\subsection{Theoretical Justification for Using h-space}\label{appendix:theoretical_justification}

In this section, we provide a limited justification for preferring
the h-space as the intervention point. The result below does not directly prove
semantic concentration; rather, it shows that a downsampling convolution
contracts i.i.d. isotropic noise by the kernel $\ell_2$ norm. Repeated
application therefore attenuates stochastic fluctuations toward the bottleneck,
helping explain why h-space is a more stable space for detection and suppression
than shallower branches.

In Section~\ref{sec:analysis}, we discussed our preference for utilizing the h-space over shallower layers, such as upsampling blocks and skip connections. In this section, we investigate the depth-wise characteristics of the U-Net architecture to support our use of the h-space. In typical diffusion models employing U-Net, downsampling is performed by convolutional layers (CNNs), halving the spatial dimensions and increasing channel depth~\cite{podell2023sdxl,lcm,arkhipkin2024kandinsky}.

\begin{proposition}
Given an $n\times n$ data matrix $X_0=\{X_0^{i,j}\}_{1\leq i,j\leq 2n}$ and an $m\times m$ convolutional kernel $M = \{\lambda_{i,j}\}_{1\leq i,j\leq m}$ used in a downsampling layer, the standard deviation of the noise added to diffused data $X_t$ (for $t \in {0,1,...,T}$) is scaled by a factor of $||M||_F=\sqrt{\sum_{i,j}\lambda_{i,j}^2}$ through the convolutional operation.
\end{proposition}

We empirically checked that the mean of kernel norms in each downsampling block are sufficiently small, effectively suppressing the added noise and confirming the theoretical prediction of Proposition~\ref{thm:hspace} (e.g., SDXL shows kernel norms of 0.2036 and 0.1509 in its respective downsampling layers). 

This inherent denoising capability supports our decision to leverage the h-space. During the detect phase, our anomaly detection approach successfully identifies irregularities within latent samples (see the second column in Figure~\ref{fig:score_plot}). However, directly correcting these anomalies in shallow layers poses challenges, as simultaneous adjustments of image content and added noise are needed to maintain the normal distribution of the score function outputs (see Equation~\ref{eq:score_loss}). Conversely, corrections applied in the h-space are less problematic, as this deeper latent space inherently condenses semantic information while effectively reducing noise. The final column of Figure~\ref{fig:component_plot} clearly illustrates that corrections in shallower layers, such as skip connections or upsampling blocks, increase score deviations in hallucinated regions, whereas corrections within the h-space significantly mitigate these deviations.

\begin{proof}

Given the diffusion process, we have:
\begin{equation}
X_t = \sqrt{\bar{\alpha}_t} X_0 + \sqrt{1 - \bar{\alpha}_t} \epsilon,
\end{equation}
where $\epsilon \sim \mathcal{N}(0, I)$ represents isotropic Gaussian noise. Let $F(\cdot): \mathbb{R}^{2n\times2n} \to \mathbb{R}^{n\times n}$ denote the convolutional downsampling operation with an $m\times m$ kernel $M = \{\lambda_{i,j}\}_{1\leq i,j\leq m}$. Due to the linearity of convolution, we have:
\begin{equation}
F(X_t) = \sqrt{\bar{\alpha}_t} F(X_0) + \sqrt{1 - \bar{\alpha}_t} F(\epsilon).
\end{equation}

Since $X_0$ remains constant throughout the diffusion process, its variance is zero, thus we focus solely on $\text{Var}(F(\epsilon))$. The elements of $F(\epsilon)$ at position $(i,j)$ are explicitly given by the convolution operation:
\begin{equation}
F(\epsilon)^{i,j} = \sum_{a=1}^{m}\sum_{b=1}^{m}\lambda_{a,b}\epsilon^{s+a, t+b}, \quad 1 \leq i,j \leq n,
\end{equation}
where $s=(i-1)\cdot 2$ and $t=(j-1)\cdot 2$ denote the starting indices for convolution with stride 2 (typical for downsampling).

Given that elements $\epsilon^{i,j}\sim\mathcal{N}(0,1)$ are independent and identically distributed (i.i.d.), we have:
\begin{align*}
\text{Var}(F(\epsilon)^{i,j}) &= \sum_{a=1}^{m}\sum_{b=1}^{m}\lambda_{a,b}^2\text{Var}(\epsilon^{s+a,t+b}) \\ &= \sum_{a=1}^{m}\sum_{b=1}^{m}\lambda_{a,b}^2 = ||M||_F^2.
\end{align*}

Hence, the standard deviation of $X_t$ after transformation by convolution operation $F(\cdot)$ is scaled by a factor of $||M||_2$, completing the proof.

\end{proof}

\subsection{Theoretical Justification for Using the Self-Attention latent in Transformers}
\label{appendix:theoretical_justification_transformer}

We next give the Transformer analogue. Conditioned on the realized attention
weights at a given layer, self-attention acts as a data-adaptive averaging
kernel whose effective norm is $\sqrt{\sum_{j=1}^N w_j^2}$. This explains why
deep self-attention latents can serve as stable intervention points, although it
does not by itself imply that any specific depth is universally optimal.

\begin{proposition}
Given tokens $x_j=s_j+\varepsilon_j\in\mathbb{R}^d$ with i.i.d.\ $\varepsilon_j\sim\mathcal{N}(0,\sigma^2 I_d)$ and conditioned on realized attention weights $w=\mathrm{softmax}((qK^\top)/\tau)$ ($w_j\ge0$, $\sum_j w_j=1$), the attention output $y=\sum_{j=1}^N w_j x_j$ scales the isotropic noise standard deviation by $\sqrt{\sum_{j=1}^N w_j^2}$, i.e.,
\[
\mathrm{Cov}(y_{\mathrm{noise}})=\sigma^2\!\left(\sum_{j=1}^N w_j^2\right) I_d,
\]
so $\sum_j w_j^2\!\in[1/N,1]$ plays the role of a \emph{data-adaptive kernel norm} (cf.\ Proposition~\ref{thm:hspace} for CNNs).
\end{proposition} 

The quantity $\sum_{j=1}^N w_j^2$ is exactly the squared $\ell_2$ norm of the attention kernel $w$, so it plays the role of a \emph{kernel norm} controlling how much i.i.d.\ isotropic noise is averaged by the attention operator. When attention is spread out (high-entropy, near-uniform weights), $\sum_j w_j^2$ is small ($\approx 1/N$), giving strong averaging and thus strong denoising. When attention is sharply peaked (nearly one-hot), $\sum_j w_j^2 \approx 1$ and noise is essentially passed through unchanged.

This establishes that, for fixed attention weights $w$, the noise covariance after self-attention is
\[
\mathrm{Cov}(y_{\mathrm{noise}}) 
= \sigma^2 \Big(\sum_{j=1}^N w_j^2\Big) I_d
\]
and hence the standard deviation is scaled by $\sqrt{\sum_{j=1}^N w_j^2}$, as claimed.

\begin{proof}
Write each token as $x_j = s_j + \varepsilon_j$ with signal $s_j \in \mathbb{R}^d$ and noise $\varepsilon_j \sim \mathcal{N}(0,\sigma^2 I_d)$ i.i.d.\ across $j$.
For fixed attention weights $w = (w_1,\dots,w_N)$ with $w_j \ge 0$ and $\sum_{j=1}^N w_j = 1$, the attention output can be decomposed as
\[
y \;=\; \sum_{j=1}^N w_j x_j 
   \;=\; \underbrace{\sum_{j=1}^N w_j s_j}_{=:y_{\mathrm{sig}}}
      \;+\; \underbrace{\sum_{j=1}^N w_j \varepsilon_j}_{=:y_{\mathrm{noise}}}.
\]
We analyze the distribution of the noise term $y_{\mathrm{noise}}$.

\paragraph{Step 1: Mean and covariance of the noise.}
Each $\varepsilon_j$ is zero-mean Gaussian with covariance $\sigma^2 I_d$, and $\{\varepsilon_j\}_{j=1}^N$ are independent. Hence
\[
\mathbb{E}[\varepsilon_j] = 0, 
\qquad
\mathrm{Cov}(\varepsilon_j) = \sigma^2 I_d, 
\qquad
\mathbb{E}[\varepsilon_j \varepsilon_k^\top] = 0 \quad (j \neq k).
\]
Therefore
\[
\mathbb{E}[y_{\mathrm{noise}}]
= \mathbb{E}\Big[\sum_{j=1}^N w_j \varepsilon_j\Big]
= \sum_{j=1}^N w_j\,\mathbb{E}[\varepsilon_j]
= 0.
\]

For the covariance, using linearity and independence,
\begin{align*}
\mathrm{Cov}(y_{\mathrm{noise}})
&= \mathbb{E}\big[\,y_{\mathrm{noise}} y_{\mathrm{noise}}^\top\,\big] \\
&= \mathbb{E}\Big[ \Big( \sum_{j=1}^N w_j \varepsilon_j \Big)
                 \Big( \sum_{k=1}^N w_k \varepsilon_k \Big)^\top \Big] \\
&= \sum_{j=1}^N \sum_{k=1}^N w_j w_k\, \mathbb{E}\big[\varepsilon_j \varepsilon_k^\top\big].
\end{align*}
By independence, $\mathbb{E}[\varepsilon_j \varepsilon_k^\top] = 0$ for $j \ne k$, and for $j=k$ we have $\mathbb{E}[\varepsilon_j \varepsilon_j^\top] = \sigma^2 I_d$. Thus only the diagonal terms remain:
\[
\mathrm{Cov}(y_{\mathrm{noise}})
= \sum_{j=1}^N w_j^2\, \sigma^2 I_d
= \sigma^2 \Big(\sum_{j=1}^N w_j^2\Big) I_d.
\]
In particular, every coordinate of $y_{\mathrm{noise}}$ has variance
$\sigma^2 \sum_{j=1}^N w_j^2$, so the noise standard deviation is scaled by
$\sqrt{\sum_{j=1}^N w_j^2}$ compared to the original $\sigma$.

\paragraph{Step 2: Range of the kernel norm $\sum_j w_j^2$.}
Since $w$ is a probability vector ($w_j \ge 0$, $\sum_j w_j = 1$), its squared $\ell_2$ norm satisfies
\[
\frac{1}{N} \;\le\; \sum_{j=1}^N w_j^2 \;\le\; 1.
\]

\emph{Lower bound.}
Apply the Cauchy--Schwarz inequality to $w$ and the all-ones vector $\mathbf{1}$:
\[
\big(\mathbf{1}^\top w\big)^2 
\le \|\mathbf{1}\|_2^2 \, \|w\|_2^2
= N \sum_{j=1}^N w_j^2.
\]
But $\mathbf{1}^\top w = \sum_{j=1}^N w_j = 1$, so
\[
1 \le N \sum_{j=1}^N w_j^2
\quad\Longrightarrow\quad
\sum_{j=1}^N w_j^2 \ge \frac{1}{N}.
\]
Equality holds iff all $w_j$ are equal, i.e.\ $w_j = 1/N$ (uniform attention).

\emph{Upper bound.}
Because $w_j \ge 0$ and $\sum_j w_j = 1$,
\[
\sum_{j=1}^N w_j^2 
\le \sum_{j=1}^N w_j \cdot \max_k w_k
= \max_k w_k
\le 1.
\]
The last inequality is tight only when $\max_k w_k = 1$, i.e.\ when the attention is one-hot ($w_j = 1$ for some $j$ and $0$ otherwise). In that case $\sum_j w_j^2 = 1$.

Thus
\[
\sum_{j=1}^N w_j^2 \in \big[1/N,\,1\big],
\]
and the noise standard deviation is scaled by a factor
\[
\sqrt{\sum_{j=1}^N w_j^2} \in \big[1/\sqrt{N},\,1\big].
\]

\end{proof}

\subsection{SNR Analysis}\label{appendix:snr_proof}

\begin{definition}
\label{def:snr}
We define the latent SNR at timestep $t$ as
\begin{equation}
\mathrm{SNR}(z_t) \;:=\;
\frac{\mathbb{E}\|s_t\|_2^2}{\mathbb{E}\|n_t\|_2^2}.
\label{eq:snr_def}
\end{equation}
\end{definition}

\paragraph{EMA as a semantic estimator.}
The following proposition formalizes why deep latents are a stable ``sweet spot'' for detect--suppress:
EMA acts as a low-pass filter that reduces noise variance while tracking slowly-varying semantics.

\begin{proposition}
\label{prop:ema_estimator}
Under Assumptions~\ref{assump:decomp}--\ref{assump:slow_drift}, the EMA satisfies
\begin{align}
\big\|\mathbb{E}[\bar z_t]-s_t\big\|_2
&\;\le\; \frac{\gamma}{1-\gamma}\,\delta,
\label{eq:ema_bias}
\\
\mathrm{Var}(\bar z_t)
&\;\le\; \frac{1-\gamma}{1+\gamma}\,\mathrm{Var}(n_t),
\label{eq:ema_var}
\end{align}
where $\mathrm{Var}(\cdot)$ denotes the coordinate-wise isotropic variance (or any consistent scalar
variance proxy).
\end{proposition}

\begin{proof}
Unrolling EMA yields
$\bar z_t = (1-\gamma)\sum_{k\ge 0}\gamma^k z_{t+k}$, hence
$\mathbb{E}[\bar z_t] = (1-\gamma)\sum_{k\ge 0}\gamma^k s_{t+k}$.
Using $\|s_{t+k}-s_t\|_2 \le k\delta$ gives \eqref{eq:ema_bias}.
For \eqref{eq:ema_var}, note that the noise term is
$(1-\gamma)\sum_{k\ge 0}\gamma^k n_{t+k}$ and sum the geometric series of variances.
\end{proof}

Let $M_t\in\{0,1\}^{d}$ denote a binary mask (broadcastable to the latent shape) produced by the detect
step. Define masked/unmasked parts by $a_{t,M}:=M_t\odot a_t$ and $a_{t,\neg M}:=(1-M_t)\odot a_t$.

We analyze a unified suppression operator:
\begin{equation*}
\hat z_t
\;=\;
(1-M_t)\odot z_t \;+\; M_t\odot\big(\kappa z_t + (1-\kappa)\tilde z_t\big),
\qquad \kappa\in[0,1],
\end{equation*}
where $\tilde z_t$ is a reference estimate. In our Transformer variant, $\tilde z_t=\bar z_t$
(EMA blending). In our U-Net implementation, the masked channels are shrunk with a channel-aware
factor, which can be interpreted as a per-channel version of \eqref{eq:generic_suppress}
(see Remark~\ref{rem:impl}).

Define the \emph{noise concentration} within the detected mask by
\begin{equation}
\eta_t \;:=\; \frac{\mathbb{E}\|n_{t,M}\|_2^2}{\mathbb{E}\|n_t\|_2^2}\in[0,1],
\label{eq:eta_def}
\end{equation}
which captures how much noise energy is covered by $M_t$.
For the pure scaling case ($\tilde z_t\equiv 0$), we also define the \emph{signal concentration}
\begin{equation}
\rho_t \;:=\; \frac{\mathbb{E}\|s_{t,M}\|_2^2}{\mathbb{E}\|s_t\|_2^2}\in[0,1].
\end{equation}

\begin{theorem}
\label{thm:snr_scaling}
Consider \eqref{eq:generic_suppress} with $\tilde z_t\equiv 0$ (i.e., $\hat z_t=(1-M_t)\odot z_t+\kappa M_t\odot z_t$).
Under Assumption~\ref{assump:decomp}, the SNR satisfies
\begin{equation}
\frac{\mathrm{SNR}(\hat z_t)}{\mathrm{SNR}(z_t)}
\;\ge\;
\frac{1-(1-\kappa^2)\rho_t}{1-(1-\kappa^2)\eta_t}.
\label{eq:snr_gain_scaling}
\end{equation}
In particular, if the mask captures proportionally more noise than signal (i.e., $\eta_t>\rho_t$),
then $\mathrm{SNR}(\hat z_t)>\mathrm{SNR}(z_t)$ for any $\kappa\in(0,1)$.
\end{theorem}

\begin{proof}
Since $M_t$ and $(1-M_t)$ have disjoint support,
$\hat z_t = s_{t,\neg M}+\kappa s_{t,M} + n_{t,\neg M}+\kappa n_{t,M}$.
By Assumption~\ref{assump:decomp}, cross terms vanish in expectation, yielding
$\mathbb{E}\|\hat s_t\|_2^2 = \mathbb{E}\|s_t\|_2^2\big(1-(1-\kappa^2)\rho_t\big)$ and
$\mathbb{E}\|\hat n_t\|_2^2 = \mathbb{E}\|n_t\|_2^2\big(1-(1-\kappa^2)\eta_t\big)$,
which implies \eqref{eq:snr_gain_scaling}.
\end{proof}

\begin{theorem}
\label{thm:snr_ema}
Consider \eqref{eq:generic_suppress} with $\tilde z_t=\bar z_{t+1}$ and define the EMA estimation error
$e_t := \bar z_{t+1}-s_t$.
Under Assumption~\ref{assump:decomp} and assuming $\mathbb{E}\langle n_t,e_t\rangle=0$,
the post-suppression SNR satisfies
\begin{equation}
\frac{\mathrm{SNR}(\hat z_t)}{\mathrm{SNR}(z_t)}
\;\ge\;
\frac{1}{1-(1-\kappa^2)\eta_t + (1-\kappa)^2 \varepsilon_t},
\qquad
\varepsilon_t := \frac{\mathbb{E}\|e_{t,M}\|_2^2}{\mathbb{E}\|n_t\|_2^2}.
\label{eq:snr_gain_ema}
\end{equation}
Consequently, $\mathrm{SNR}(\hat z_t)>\mathrm{SNR}(z_t)$ whenever
\begin{equation}
(1-\kappa^2)\eta_t \;>\; (1-\kappa)^2 \varepsilon_t.
\label{eq:snr_improve_condition}
\end{equation}
\end{theorem}

\begin{proof}
Inside the mask,
$\hat z_{t,M} = \kappa z_{t,M} + (1-\kappa)\bar z_{t+1,M}
= s_{t,M} + \kappa n_{t,M} + (1-\kappa)e_{t,M}$,
so the \emph{signal} is preserved while the noise is shrunk by $\kappa$ plus an EMA error term.
Therefore
\[
\hat z_t - s_t
= n_{t,\neg M} + \kappa n_{t,M} + (1-\kappa)e_{t,M}.
\]
Taking squared norms and expectations, disjoint support removes cross terms between $n_{t,\neg M}$
and $n_{t,M}$, and the orthogonality assumption removes the cross term between $n_t$ and $e_t$.
Thus
\[
\mathbb{E}\|\hat z_t - s_t\|_2^2
=
\mathbb{E}\|n_t\|_2^2\big(1-(1-\kappa^2)\eta_t\big)
+ (1-\kappa)^2\mathbb{E}\|e_{t,M}\|_2^2,
\]
which implies \eqref{eq:snr_gain_ema} and \eqref{eq:snr_improve_condition}.
\end{proof}

\begin{corollary}
\label{cor:score_stability}
Let the remaining mapping from the target latent to noise prediction be
$\epsilon_\theta(\cdot,t)=g_t(\cdot)$ and assume $g_t$ is $L_t$-Lipschitz.
Then the change in the predicted score satisfies
\begin{equation}
\|\; s_\theta(t,\hat x_t) - s_\theta(t,x_t)\;\|_2
\;\le\;
\frac{L_t}{\sigma_t}\,\|\hat h_t-h_t\|_2
\;\propto\;
\frac{L_t(1-\kappa)}{\sigma_t}\,\|M_t\odot(z_t-\tilde z_t)\|_2,
\label{eq:score_stability}
\end{equation}
explaining why suppressing large residuals in deep latents reduces abnormal spikes in score dynamics.
\end{corollary}

The percentile $p$ controls the mask size (hence $\eta_t$), while $\kappa$ controls shrinkage strength.
Theorem~\ref{thm:snr_scaling} suggests that, for pure scaling, SNR improves when the mask is
noise-dominant ($\eta_t>\rho_t$). Theorem~\ref{thm:snr_ema} further shows that EMA blending is
signal-preserving and improves SNR as long as the EMA error on masked entries is small
(\eqref{eq:snr_improve_condition}), which is precisely encouraged by deep-latent ``sweet spots'' where
EMA is accurate (Prop.~\ref{prop:ema_estimator}).

\begin{remark}
\label{rem:impl}
Eq.~\eqref{eq:generic_suppress} exactly matches our Transformer suppression
(EMA blending on masked features).
Our U-Net suppression applies a channel-aware shrinkage on the masked entries; it can be viewed as a
per-channel variant of \eqref{eq:generic_suppress} (with $\kappa$ replaced by a diagonal matrix) and
inherits the same SNR intuition: if the detected subset concentrates noise energy, targeted shrinkage
improves the global latent SNR.
\end{remark}

In the outlier regime selected by our detector, the following 
lemma shows that replacing masked EMA blending with a masked 
channel-wise scaling is justified: the only discrepancy consists 
of (i)~a coefficient mismatch $\tilde\kappa-\alpha_t$, 
which we empirically confirm is near zero at our chosen 
hyperparameters(averaged 0.066 in SDXL and 0.074 in LCM), and 
(ii)~a provably small EMA residual suppressed by the detection 
threshold $\lambda$.

\begin{lemma}
\label{lem:naive_ema_vs_alpha_scaling}
Let $\mathbf z_t := -\mathbf h_t/\sigma_t$ and define the EMA as
\[
\bar{\mathbf z}_t=\gamma \bar{\mathbf z}_{t+1}+(1-\gamma)\mathbf z_t.
\]
Let the anomaly mask be defined elementwise by
\[
\Delta=\log\!|\frac{\mathbf z_t}{\bar{\mathbf z}_{t+1}}|,\qquad
M_t=(\Delta>\lambda),
\]
(assuming the elementwise log-ratio is well-defined under the same numerical
stabilization used in practice).
Consider the \emph{naive EMA blending} operator:
\[
\mathbf u_t^{\mathrm{EMA}} := \kappa \mathbf z_t + (1-\kappa)\bar{\mathbf z}_t,\qquad \kappa\in[0,1].
\]
For U-Net, define the channel statistic $n_t\in\mathbb R^{C\times 1\times 1}$ 
and let the channel-aware scaling coefficient be
\[
\alpha_t := \kappa\, n_t
\quad \text{(broadcastable to the latent shape).}
\]
Define the \emph{U-Net scaling} output in the scaled-latent space as
\[
\mathbf u_t^{\mathrm{UNet}} := \alpha_t \odot \mathbf z_t
\quad\Big(\text{equivalently, } \hat{\mathbf h}_t=\kappa(n_t\cdot \mathbf h_t)
\Leftrightarrow \hat{\mathbf z}_t=\kappa(n_t\cdot \mathbf z_t)\Big).
\]

Let $\kappa_{\mathrm{eff}} := 1-\gamma(1-\kappa)=\kappa+(1-\kappa)(1-\gamma)$.
Then for every index $i$ such that $(M_t)_i=1$,
\[
\big|(\mathbf u_t^{\mathrm{EMA}})_i - (\mathbf u_t^{\mathrm{UNet}})_i\big|
\;\le\;
\Big(|\kappa_{\mathrm{eff}}-\alpha_{t,i}| + (1-\kappa)\gamma e^{-\lambda}\Big)\,|\mathbf z_{t,i}|.
\]
Equivalently, for any $p\in[1,\infty]$,
\[
\|M_t\odot(\mathbf u_t^{\mathrm{EMA}}-\mathbf u_t^{\mathrm{UNet}})\|_p
\;\le\;
\|M_t\odot(\kappa_{\mathrm{eff}}-\alpha_t)\odot \mathbf z_t\|_p
\;+\;
(1-\kappa)\gamma e^{-\lambda}\,\|M_t\odot \mathbf z_t\|_p.
\]
\end{lemma}

\begin{proof}
Expand the naive EMA blending using the EMA recursion:
\[
\mathbf u_t^{\mathrm{EMA}}
=\kappa \mathbf z_t + (1-\kappa)\big(\gamma \bar{\mathbf z}_{t+1}+(1-\gamma)\mathbf z_t\big)
=\underbrace{(1-\gamma(1-\kappa))}_{=\kappa_{\mathrm{eff}}}\mathbf z_t
+(1-\kappa)\gamma \bar{\mathbf z}_{t+1}.
\]
Subtract $\mathbf u_t^{\mathrm{UNet}}=\alpha_t\odot \mathbf z_t$ to obtain
\[
\mathbf u_t^{\mathrm{EMA}}-\mathbf u_t^{\mathrm{UNet}}
=(\kappa_{\mathrm{eff}}-\alpha_t)\odot \mathbf z_t
+(1-\kappa)\gamma \bar{\mathbf z}_{t+1}.
\]
For $(M_t)_i=1$, the mask definition implies $\Delta_i>\lambda$, hence
$\big|\bar{\mathbf z}_{t+1,i}\big|\le e^{-\lambda}\big|\mathbf z_{t,i}\big|$.
Applying the triangle inequality yields the elementwise bound, and taking a masked
$\ell_p$ norm gives the second inequality.
\end{proof}

We keep a unified \emph{detection} principle across backbones, but use backbone-specific \emph{suppression} operators. For U-Net h-space, we justify our channel-aware scaling by showing that selective shrinkage on detected low-SNR subsets improves the \emph{global} latent SNR, and that channel-dependent gains are preferable to uniform scaling. 

\begin{assumption}
\label{assump:unet_decomp} 
For a fixed timestep $t$ in the detect phase, let the target h-space latent be $h_t \in \mathbb{R}^{C\times H\times W}$. For each channel $c$, we write 
\begin{equation} h_{t,c,u} = s_{t,c,u} + \varepsilon_{t,c,u}, \qquad u \in \{1,\dots,m\}, \;\; m = HW, 
\end{equation}
where $s_{t,c,u}$ is the semantic component and $\varepsilon_{t,c,u}$ is a stochastic component. We assume $\mathbb{E}[\varepsilon_{t,c,u}] = 0$, independence across spatial sites $u$ within each channel, and $\mathrm{Var}(\varepsilon_{t,c,u}) = \sigma_c^2$. 
\end{assumption} 

\begin{proposition} 
\label{prop:channel_mean_proxy} 
Define the spatial channel mean 
\begin{equation} 
\bar h_{t,c} := \frac{1}{m}\sum_{u=1}^{m} h_{t,c,u}, \qquad \mu_{t,c} := \frac{1}{m}\sum_{u=1}^{m} s_{t,c,u}. 
\end{equation} 
Under Assumption~\ref{assump:unet_decomp}, 
\begin{equation} 
\mathbb{E}[\bar h_{t,c}] = \mu_{t,c}, \qquad \mathrm{Var}(\bar h_{t,c}) = \frac{\sigma_c^2}{m}. 
\end{equation} 
Hence, the channel statistic 
\begin{equation} 
n_{t,c} := |\bar h_{t,c}| 
\end{equation} 
is a low-variance proxy for the structured channel activation in low-resolution h-space. 
\end{proposition} 

\begin{proof} 
By linearity of expectation, $\mathbb{E}[\bar h_{t,c}] = \frac{1}{m}\sum_u \mathbb{E}[s_{t,c,u}+\varepsilon_{t,c,u}] = \mu_{t,c}$. Since the noise is zero-mean and independent across spatial sites, 
\begin{equation} \mathrm{Var}(\bar h_{t,c}) = \mathrm{Var}\!\left(\frac{1}{m}\sum_{u=1}^{m}\varepsilon_{t,c,u}\right) = \frac{1}{m^2}\sum_{u=1}^{m}\sigma_c^2 = \frac{\sigma_c^2}{m}. 
\end{equation} 
Thus spatial averaging suppresses stochastic fluctuations while preserving the channel-wise semantic trend. 
\end{proof} 

\begin{theorem}
\label{thm:unet_global_snr} 
Let $M_t \in \{0,1\}^{C\times H\times W}$ be the anomaly mask from the detect step. For each channel $c$, define the masked index set \begin{equation} 
\Omega_c := \{u : M_{t,c,u}=1\}. 
\end{equation} 
Consider the U-Net suppression operator 
\begin{equation} 
\hat h_{t,c,u} = (1-M_{t,c,u})h_{t,c,u} + M_{t,c,u} g_c h_{t,c,u}, \qquad g_c \in [0,1]. 
\label{eq:masked_channel_gain} 
\end{equation} 
Let 
\begin{align} S_0 &:= \sum_{(c,u):\, M_{t,c,u}=0} s_{t,c,u}^2, & N_0 &:= \mathbb{E}\sum_{(c,u):\, M_{t,c,u}=0} \varepsilon_{t,c,u}^2, \\ S_c &:= \sum_{u\in\Omega_c} s_{t,c,u}^2, & N_c &:= \mathbb{E}\sum_{u\in\Omega_c} \varepsilon_{t,c,u}^2. 
\end{align} 
Then the post-suppression global latent SNR is 
\begin{equation} \mathrm{SNR}(\hat h_t) = \frac{S_0 + \sum_c g_c^2 S_c}{N_0 + \sum_c g_c^2 N_c}. 
\label{eq:global_snr_after} 
\end{equation} 
Furthermore: \textbf{(i) Uniform shrinkage on a low-SNR detected subset improves global SNR.} If $g_c=a$ for all masked channels with $a\in[0,1)$, then 
\begin{equation} \mathrm{SNR}(\hat h_t^{(a)}) > \mathrm{SNR}(h_t) \quad\Longleftrightarrow\quad \frac{\sum_c S_c}{\sum_c N_c} < \frac{S_0}{N_0}. 
\label{eq:low_snr_subset_condition} 
\end{equation} 
That is, shrinking the detected subset improves the global latent SNR whenever the detected subset has lower SNR than the unmasked subset. \textbf{(ii) Channel-aware gains outperform uniform shrinkage when aligned with channel SNR.} Define the per-channel masked SNR \begin{equation} 
r_c := \frac{S_c}{N_c}, \qquad p_c := \frac{N_c}{\sum_j N_j}. \end{equation} 
Let $\beta := \sum_c p_c g_c^2$ be the average suppression budget, and let $\hat h_t^{\mathrm{uni}}$ denote the uniform-gain baseline with gain $\sqrt{\beta}$ on all masked channels. Then 
\begin{equation} \mathrm{SNR}(\hat h_t^{\mathrm{ca}}) - \mathrm{SNR}(\hat h_t^{\mathrm{uni}}) = \frac{\sum_c N_c}{\,N_0+\beta\sum_c N_c\,} \cdot \mathrm{Cov}_{p}(g_c^2, r_c). 
\label{eq:covariance_gain} 
\end{equation} 
Hence, if $g_c^2$ is positively correlated with the per-channel masked SNR $r_c$, channel-aware suppression strictly dominates uniform shrinkage under the same average correction strength. 
\end{theorem} 

\begin{proof} 
Eq.~\eqref{eq:global_snr_after} follows by decomposing the total signal and noise energies inside and outside the detected mask and noting that masked entries are multiplied by $g_c$. For part (i), \begin{equation} 
\mathrm{SNR}(\hat h_t^{(a)}) = \frac{S_0 + a^2 \sum_c S_c}{N_0 + a^2 \sum_c N_c}. 
\end{equation} 
Comparing this with $\mathrm{SNR}(h_t)=\frac{S_0+\sum_c S_c}{N_0+\sum_c N_c}$ and cross-multiplying, we obtain 
\begin{equation} 
\mathrm{SNR}(\hat h_t^{(a)}) > \mathrm{SNR}(h_t) \iff (1-a^2)\!\left(S_0\sum_c N_c - N_0\sum_c S_c\right) > 0, 
\end{equation} 
which is equivalent to \eqref{eq:low_snr_subset_condition}. For part (ii), by the definition of $\beta$, 
\begin{equation} 
\sum_c g_c^2 N_c = \beta \sum_c N_c, 
\end{equation} 
so both $\hat h_t^{\mathrm{ca}}$ and $\hat h_t^{\mathrm{uni}}$ have the same denominator. Their numerator difference is 
\begin{equation} 
\sum_c g_c^2 S_c - \beta \sum_c S_c = \sum_c N_c \left(g_c^2 r_c - \beta r_c\right) = \left(\sum_c N_c\right)\mathrm{Cov}_{p}(g_c^2,r_c), \end{equation} 
which yields \eqref{eq:covariance_gain}. 
\end{proof} 

\begin{corollary}
\label{cor:practical_unet_gain} 
Assume the implemented gain satisfies $g_c=\kappa n_{t,c}\in[0,1]$ on the detected subset, where $n_{t,c}=|\mathrm{mean}_{\mathrm{spatial}}(h_{t,c})|$. If $n_{t,c}$ is positively associated with the per-channel masked SNR $r_c$, then the U-Net suppression 
\begin{equation} 
\hat h_t = (1-M_t)\odot h_t + \kappa M_t \odot (n_t \cdot h_t) \end{equation} 
preferentially preserves signal-dominant channels while downweighting noise-dominant channels, thereby improving the global latent SNR more than uniform masked scaling. 
\end{corollary} 

\begin{remark} 
The theorem does \emph{not} claim that scaling a single channel increases that channel's own SNR. Rather, it shows that DUNE improves the \emph{global latent SNR} by selectively downweighting the detected low-SNR subset and by preserving high-SNR channels more aggressively than low-SNR channels. This matches the role of channel-wise suppression in U-Net h-space. 
\end{remark}


\section{User Survey Description}\label{appendix:usersurvey}

Each participant completed all pairwise comparisons. For each question, the left/right presentation order was randomized, and the order of prompts/models was shuffled. Participants were asked to choose the better image overall in terms of visual quality and artifact reduction. We report the aggregate preference over all judgments.

Prompts used for SDXL:

\begin{enumerate}
\item \textit{A cheerful picnic scene in a sunny meadow with colorful flowers, happy people, and a vibrant sky, photorealistic, 8k.}
\item \textit{A person holding a mirror that perfectly reflects a different scene, not what is in front of them, highly detailed, 8k.}
\item \textit{A grand ballroom from the Victorian era, chandeliers glowing, people in elegant attire, rich details, 8k.}
\item \textit{An Escher-style staircase where people walk in all directions, physically impossible architecture, highly detailed, 8k.}
\item \textit{A fantasy creature: a hybrid between a fox and a phoenix, with fiery tails glowing against a twilight background, magical realism.}
\item \textit{A peaceful snowy village under the northern lights, warm lights glowing from cottage windows, photorealistic, 8k.}
\item \textit{A grand city built inside a massive crystal cavern, with light refracting in every direction, breathtaking spectacle, 8k.}
\item \textit{A magical marketplace hidden in a forest, stalls selling potions, enchanted items, and rare creatures, whimsical and lively, 8k.}
\item \textit{A grand steampunk city with airships floating among towering brass structures, intricate details, 8k.}
\item \textit{A grand temple floating among the clouds, partially hidden by mist, with golden sunlight filtering through, ethereal atmosphere, highly detailed, 8k.}
\end{enumerate}

Prompts used for Kandinsky 3:

\begin{enumerate}
\item \textit{A futuristic robot chef preparing a meal in a sleek modern kitchen, cyberpunk aesthetic, ultra-detailed, 8k.}
\item \textit{A medieval blacksmith forging a glowing sword in a dimly lit forge, sparks flying, cinematic lighting, 8k.}
\item \textit{A person holding a mirror that perfectly reflects a different scene, not what is in front of them, highly detailed, 8k.}
\item \textit{A world where gravity works in reverse, people and objects floating upward while birds walk on the ground, highly detailed, 8k.}
\item \textit{A hand writing a letter, with the actual readable text visible and correctly spelled, hyper-realistic, 8k.}
\item \textit{A grand temple floating among the clouds, partially hidden by mist, with golden sunlight filtering through, ethereal atmosphere, highly detailed, 8k.}
\item \textit{A group of people playing chess in zero gravity, the pieces floating and moving in a realistic way, highly detailed, 8k.}
\item \textit{A grand steampunk city with airships floating among towering brass structures, intricate details, 8k.}
\item \textit{A giant enchanted tree in the middle of a glowing fairy forest, mysterious and magical ambiance, 8k.}
\item \textit{A celestial palace floating among the clouds, golden spires shining under a twilight sky, ethereal atmosphere, 8k.}
\end{enumerate}

\begin{figure*}[ht!]
\centering
\includegraphics[width=0.8\linewidth]{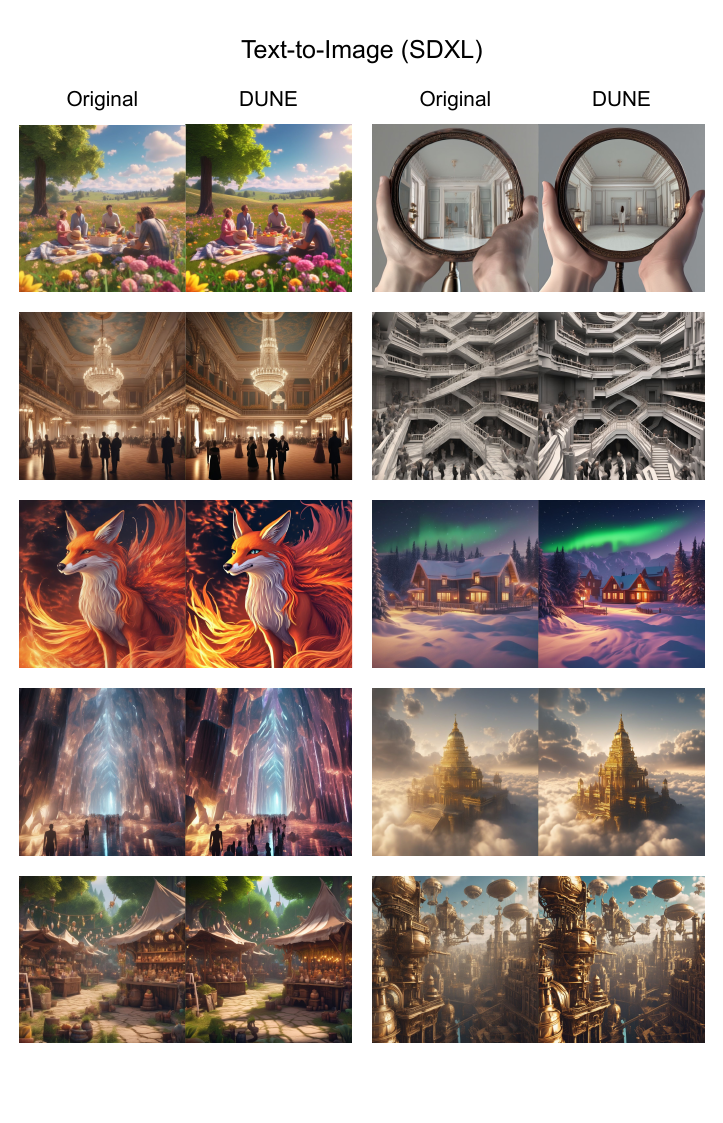}
\caption{Figures for User Survey. SDXL is used to generate these images.}
\end{figure*}

\begin{figure*}[ht!]
\centering
\includegraphics[width=0.8\linewidth]{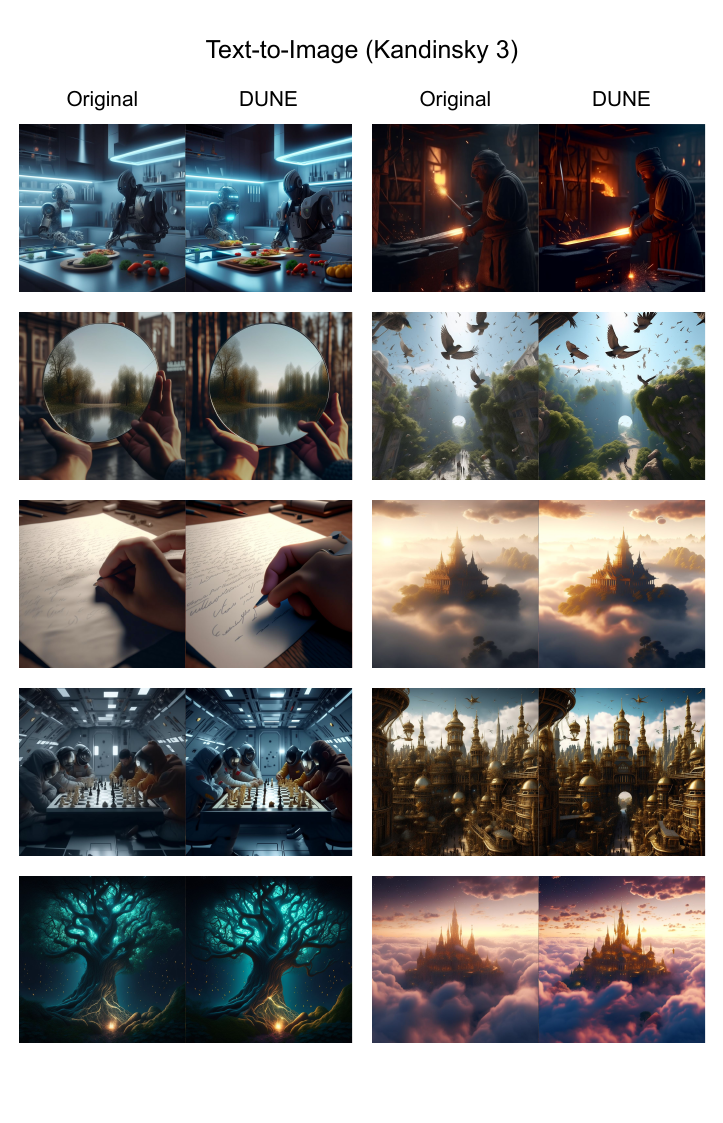}
\caption{Figures for User Survey. Kandinsky 3 is used to generate these images.}
\end{figure*}

\begin{figure*}[ht!]
\centering
\includegraphics[width=0.8\linewidth]{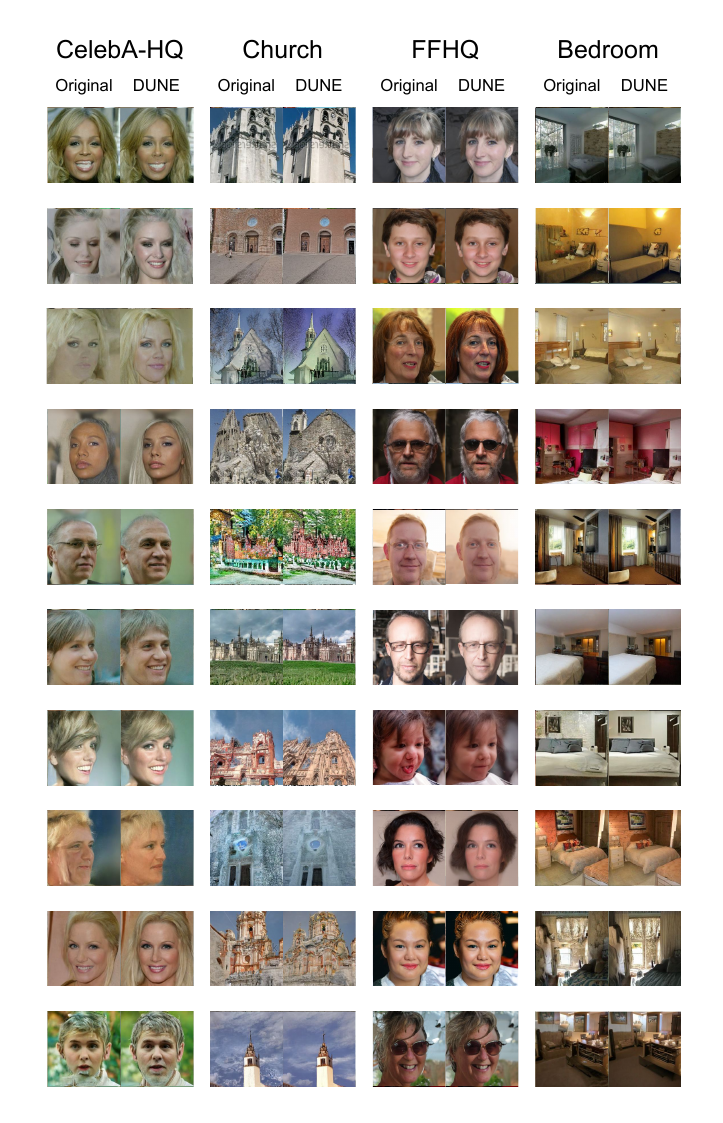}
\caption{Figures for User Survey. DDPM is used for CelebA-HQ and Church, while VE is used for FFHQ and Bedroom.}
\end{figure*}

\vspace{-0.5em}

\end{document}